\documentclass{article}

\usepackage{arxiv}
\usepackage[utf8]{inputenc} 
\usepackage[T1]{fontenc}    
\usepackage{hyperref}       
\usepackage{url}            
\usepackage{booktabs}       
\usepackage{amsfonts}       
\usepackage{nicefrac}       
\usepackage{microtype}      
\usepackage{lipsum}
\usepackage{graphicx}
\usepackage{apacite}
\usepackage{float}
\usepackage{amsmath, amssymb, amsthm}
\usepackage{algorithm}
\usepackage{algpseudocode}
\theoremstyle{definition}

\theoremstyle{plain}
\newtheorem{lemma}{Lemma}[section]
\newtheorem{theorem}{Theorem}[section]
\newtheorem{corollary}{Corollary}[section]
\newtheorem{assumption}{Assumption}[section]
\theoremstyle{remark}
\newtheorem{remark}{Remark}[section]
\graphicspath{ {./images/} }
\usepackage{newunicodechar}
\newunicodechar{∞}{\infty}

\usepackage{tikz}
\usetikzlibrary{arrows.meta,positioning,calc}
\usepackage{enumitem}
\setlist[enumerate]{itemsep=0.5ex, topsep=0.5ex, parsep=0.5ex, partopsep=0ex}

\graphicspath{ {./images/} }

\title{Quantum Topological Graph Neural Networks for Detecting Complex Fraud Patterns
}

\author{
  Mohammad Doost\textsuperscript{*} \\
  Department of Electrical Engineering\\
  Sharif University of Technology \\
  Tehran, Iran \\
  \texttt{mohammad.doost@alum.sharif.edu} \\
   \And
  Mohammad Manthouri\\
  Department of Electrical and Electronic Engineering \\
  Shahed University\\
  Tehran, Iran \\
  \texttt{mmanthouri@shahed.ac.ir} \\
}

\begin{document}
\maketitle
\begin{abstract}
We propose a novel QTGNN framework for detecting fraudulent transactions in large-scale financial networks. By integrating quantum embedding, variational graph convolutions, and topological data analysis, QTGNN captures complex transaction dynamics and structural anomalies indicative of fraud. The methodology includes quantum data embedding with entanglement enhancement, variational quantum graph convolutions with non-linear dynamics, extraction of higher-order topological invariants, hybrid quantum-classical anomaly learning with adaptive optimization, and interpretable decision-making via topological attribution. Rigorous convergence guarantees ensure stable training on noisy intermediate-scale quantum (NISQ) devices, while stability of topological signatures provides robust fraud detection. Optimized for NISQ hardware with circuit simplifications and graph sampling, the framework scales to large transaction networks. Simulations on financial datasets, such as PaySim and Elliptic, benchmark QTGNN against classical and quantum baselines, using metrics like ROC-AUC, precision, and false positive rate. An ablation study evaluates the contributions of quantum embeddings, topological features, non-linear channels, and hybrid learning. QTGNN offers a theoretically sound, interpretable, and practical solution for financial fraud detection, bridging quantum machine learning, graph theory, and topological analysis.
\end{abstract}

\keywords{Quantum Graph Neural Networks \and Topological Data Analysis \and Persistent Homology \and Convergence Guarantees \and Financial Fraud Detection}

\section{Introduction}
\label{sec:Intro}

The rapid expansion of digital financial ecosystems has created unprecedented opportunities for innovation in banking, payments, and decentralized finance, fundamentally transforming how individuals and institutions exchange value \cite{ali2022financial}. At the same time, this transformation has generated increasingly complex, high-dimensional, and large-scale transaction networks, in which fraudulent behaviors can propagate with remarkable speed and sophistication \cite{awoyemi2017credit}. Activities such as money laundering, identity theft, coordinated fraud rings, and synthetic account generation pose persistent challenges for regulators and financial institutions, leading to substantial economic losses and reputational damage. Detecting these anomalies is further complicated by the sheer volume, velocity, and diversity of transaction data, which include temporal, relational, and multimodal features. Theoretical advances in machine learning, quantum computation, and topological data analysis now offer promising avenues for building systems that are not only more accurate but also robust, interpretable, and scalable in high-stakes financial decision-making environments \cite{bin2022review}. Consequently, the convergence of quantum machine learning, graph-based representations, and topological methods is emerging as a fertile research frontier with interdisciplinary implications spanning computer science, finance, regulatory technology, and cybersecurity \cite{varmedja2019credit}.

Despite considerable progress in graph neural networks (GNNs) for fraud detection, existing approaches exhibit notable limitations when applied to real-world financial transaction networks \cite{cheng2025graph}. Traditional GNNs predominantly rely on local message-passing mechanisms, which aggregate information from immediate neighbors but often fail to capture higher-order, long-range dependencies and subtle non-linear correlations that are critical for identifying coordinated fraudulent behaviors \cite{motie2024financial}. These models also typically demand large quantities of labeled data for effective training, a requirement that is rarely satisfied in fraud detection due to highly imbalanced classes and constantly evolving adversarial strategies \cite{khaniki2025class}. In addition, most deep learning models remain opaque, providing limited interpretability, which poses regulatory and operational challenges, as institutions and auditors require clear explanations for decisions impacting customers and stakeholders \cite{alarfaj2024enhancing}. The combination of these technical and practical shortcomings creates an urgent need for models that balance expressive representational power with theoretical rigor, interpretability, and computational scalability, enabling timely and reliable detection of complex, distributed fraudulent patterns across dynamic financial networks\cite{dong2024dynamic}.

Recent years have witnessed significant exploration of advanced computational paradigms for anomaly detection beyond classical deep learning frameworks \cite{bhowmik2024quantum}. Quantum machine learning, for instance, has emerged as a promising approach capable of representing and processing high-dimensional data in ways unattainable by conventional methods, with potential exponential speedups in embedding and optimization tasks \cite{hdaib2024quantum}. Variational quantum circuits have demonstrated applicability in classification, clustering, and pattern recognition problems, while hybrid quantum–classical architectures offer practical solutions tailored for resource-limited, NISQ devices \cite{umeda2019topological}. Parallelly, topological data analysis (TDA) has gained attention for its ability to extract invariant features that characterize the global structure of data, such as persistence homology, which remain stable under noise or minor perturbations \cite{monkam2024topological}. Studies integrating graph-based learning with topological descriptors have shown improved detection of community structures, motifs, and anomalous patterns in diverse networks, including social, transportation, and power-grid systems \cite{eroglu2023topological}. However, most existing approaches address quantum learning or topological descriptors in isolation, neglecting the potential synergistic benefits of combining these methodologies to exploit both the expressive power of quantum embeddings and the robustness of topological invariants in networked data \cite{li2025hqrnn}.

This gap underscores the opportunity and necessity for a unified framework that synergistically integrates quantum computing, topological invariants, and graph neural architectures for anomaly detection in financial transaction networks. While existing quantum graph neural network models offer intriguing theoretical foundations, they often lack rigorous mechanisms for interpretability and robustness when confronted with real-world noise and adversarial manipulations \cite{barberi2025topological}. Similarly, although topological descriptors have been employed to enhance GNN performance, their integration with quantum learning has been minimally explored, particularly for large-scale financial applications where both precision and explainability are critical. Addressing these challenges is not merely a technical advancement; it carries regulatory and societal significance, as financial institutions require models that can reliably operate under adversarial conditions, adhere to computational constraints, and provide audit-ready explanations for their predictions. Furthermore, the combined use of quantum embeddings and topological analysis has the potential to uncover subtle, higher-order relational patterns that classical methods may overlook, enabling earlier and more accurate detection of complex fraud schemes. By establishing a framework that is both theoretically rigorous and practically deployable, this research aims to set a new benchmark for intelligent, transparent, and scalable anomaly detection in the financial domain \cite{innan2024financial}.

In this paper, we propose an advanced QTGNN for anomaly detection in financial transaction networks. Our framework operationalizes a five-stage methodology integrating quantum state embedding, variational graph convolutions, topological invariant extraction, hybrid anomaly learning, and interpretable decision-making. The model is designed to balance theoretical rigor with practical feasibility by optimizing for NISQ devices through circuit simplification and graph sampling. The key contributions of this work are summarized as follows:

\begin{itemize}
    \item \textbf{Leveraging quantum embeddings for richer relational representation:} Quantum embeddings encode complex, high-dimensional relational information that classical methods often struggle to capture. They exploit quantum superposition and entanglement to represent multiple transaction pathways and interdependencies simultaneously, enabling detection of subtle patterns and correlations distributed across the network, which are difficult for classical GNNs relying on local message-passing. Furthermore, quantum embeddings can reduce the dimensionality of relational data while preserving critical structural information, enhancing both computational efficiency and expressiveness.

    \item \textbf{Incorporating topological invariants for robustness and interpretability:} Topological descriptors, such as persistent homology, capture global structural features that remain invariant to noise or minor perturbations. By integrating these invariants with quantum embeddings, the learned representations become sensitive to meaningful patterns while robust against spurious connections. This approach enhances interpretability by providing insights into which network structures or higher-order motifs contribute to anomaly detection decisions, supporting transparency and regulatory compliance in financial systems.

    \item \textbf{Ensuring scalability and practicality on NISQ devices:} Although quantum models offer theoretical advantages, practical deployment is constrained by hardware limitations such as qubit decoherence, gate errors, and limited circuit depth. Our framework addresses these challenges through NISQ-aware design, employing circuit optimization, error mitigation, and hybrid quantum-classical computation strategies. This ensures that quantum-enhanced graph learning is feasible, scalable, and deployable in real-world financial transaction networks.
\end{itemize}

The remainder of this paper is organized as follows: Section ~\ref{sec:Related_Work} reviews related work in quantum learning, GNNs, and topological methods; Section ~\ref{sec:Methodology} details the proposed methodology; Section ~\ref{sec:ExperimentalValidation} presents experimental design and results; and Section ~\ref{sec:Conclusion} concludes with future research directions.

\section{Related Work}
\label{sec:Related_Work}

\subsection{Foundational Work in Graph-Based Anomaly Detection}

Early research on fraud and anomaly detection in transaction systems relied on rule-based methods, statistical models, and supervised learning algorithms \cite{akoglu2015graph}. These approaches typically depended on hand-crafted features such as transaction frequency, velocity, or amount thresholds. While effective in structured and controlled environments, they failed to capture the dynamic and relational nature of financial systems. The introduction of graph-based learning marked a paradigm shift, as it enabled modeling entities and their interactions as networks \cite{eberle2007anomaly}. Classical graph-theoretic methods, such as community detection, spectral clustering, and random walk-based approaches, provided new perspectives for identifying irregularities \cite{xiong2022analysis}. However, their reliance on shallow descriptors limited scalability and adaptability, particularly when adversaries actively disguised fraudulent patterns. These foundational studies nevertheless established the importance of relational modeling and set the stage for the integration of GNNs into fraud detection pipelines \cite{pourhabibi2020fraud}. \cite{irani2025time} comprehensively reviews time series embedding techniques for classification tasks, categorizing methods into distance-based, transformation-based, and deep learning approaches while evaluating their strengths, limitations, and suitability across diverse domains to guide effective feature representation in time series analysis. \cite{nikzat2025artificial} explores how artificial intelligence drives innovation and competitive advantage across business functions, highlighting key AI applications in marketing, operations, finance, and strategy while emphasizing the need for ethical implementation and workforce adaptation to fully realize AI’s transformative potential. \cite{nikzat2025review} reviews the transformative role of artificial intelligence in reshaping business and management, arguing that AI-driven innovation, data-driven decision-making, and process automation are key sources of strategic competitive advantage while stressing the importance of ethical governance and organizational readiness for successful AI adoption.

The shift to graph-based methods was motivated by their ability to capture complex dependencies within financial systems, such as interactions between accounts, merchants, and payment flows. Community detection algorithms, such as modularity-based approaches, were used to identify clusters of entities that might indicate coordinated fraudulent activities, such as money laundering or account takeovers \cite{coppolino2015use}. Spectral clustering, leveraging the eigenvalues of graph Laplacian matrices, partitioned transaction networks to detect anomalies as outliers in the resulting clusters \cite{larik2011clustering}. Random walk-based methods, inspired by algorithms like PageRank, ranked nodes based on their connectivity patterns, flagging those with anomalous influence or connectivity as potential fraud indicators \cite{vilella2025weirdnodes}. Despite these innovations, early graph-based methods faced significant limitations. Their static representations struggled to model temporal dynamics, such as sudden changes in transaction behavior typical of fraud \cite{zheng2018transaction}. Efforts to incorporate temporal information, such as time-weighted graphs or sliding window techniques, improved detection but introduced computational complexity and sensitivity to noise \cite{kallush2014quantum}. Furthermore, these methods often required extensive feature engineering, relying on domain expertise to define effective graph properties, which limited their adaptability to evolving fraud tactics. These challenges underscored the need for models that could learn relational and temporal patterns directly from data, paving the way for GNN-based approaches. \cite{mashhadi2025interpretable} builds an interpretable Gradient Boosting-SHAP framework that accurately predicts startup funding, patenting, and exits, identifying founder experience, past success, sector, and location as the key drivers of entrepreneurial outcomes. \cite{heravi2025vehicle} proposes a lightweight deep learning system using smartphone inertial sensors (accelerometers and gyroscopes) to reliably detect vehicle intrusions into highway work zones in real time, achieving over 95\% accuracy with low computational cost and enabling proactive safety alerts for construction workers. \cite{mohammadagha2025hybrid} develops a hybrid machine learning meta-model that combines XGBoost, Random Forest, and LightGBM predictions via stacking to accurately assess the structural condition of urban underground pipes, outperforming individual models and traditional methods in predicting deterioration grades using inspection data.

\subsection{Advances in Graph Neural Networks for Anomaly Detection}

GNNs have become the state-of-the-art for analyzing transaction networks, offering the ability to learn hierarchical representations of entities and relationships through message passing and aggregation. Methods such as Graph Convolutional Networks (GCNs), Graph Attention Networks (GATs) \cite{yilmaz2025novel}, and Graph Autoencoders have been applied to fraud detection with promising results \cite{pourhabibi2020fraud}. These models excel at capturing local neighborhood dependencies and differentiating between normal and abnormal transaction flows. Recent works have extended GNNs to dynamic graphs, enabling detection in temporally evolving networks such as blockchain transactions and real-time payment systems \cite{sha2025detecting}. For example, temporal GNNs incorporate time-stamped edges and node features to model evolving transaction patterns, improving detection of anomalies in high-velocity payment systems. Despite these advances, challenges remain. Most GNN-based models struggle with imbalanced datasets, where fraudulent transactions represent a small minority. Additionally, they often lack interpretability, limiting their adoption in regulatory contexts where model transparency is essential. The reliance on large volumes of labeled data further restricts their applicability in zero-day or evolving fraud scenarios \cite{asiri2025graph}.

Recent efforts have sought to address these limitations through innovative GNN architectures and training strategies. For instance, self-supervised learning approaches, such as contrastive learning on graphs, have been proposed to reduce the dependency on labeled data by learning robust node representations from unlabeled transaction networks \cite{wu2021learning}. These methods leverage structural and temporal patterns to distinguish normal from anomalous behaviors without requiring extensive labeled datasets. Additionally, attention-based mechanisms in GATs have been extended to incorporate edge weights and temporal dynamics, enabling more nuanced modeling of transaction flows \cite{vrahatis2024graph}. However, the computational cost of these advanced GNN models remains a barrier, particularly for real-time applications in large-scale financial networks. Moreover, interpretability remains a critical challenge, as GNNs often produce complex embeddings that are difficult to explain to stakeholders in regulatory or financial domains. These gaps highlight the need for frameworks that combine the expressive power of GNNs with mechanisms for interpretability and scalability. This survey \cite{birashk2025federated} provides a comprehensive overview of Federated Continual Learning (FCL) concepts, taxonomies, and challenges, with concise reviews of algorithms for task-incremental and class-incremental settings, distinguishing between identical and non-identical task sequences across clients to guide future distributed adaptive learning research. \cite{mojtahedi2025max} provides the first empirical evidence of the “MAX effect” — the anomalous negative relation between stocks with extreme positive daily returns and future returns — in the Swedish stock market, showing that this lottery-like stock premium is largely explained by a surge in investor sentiment during high-sentiment periods. \cite{mashhadi2025return} presents the persistence of several well-known return anomalies (momentum, value, size, profitability, and investment) in the Iranian stock market despite severe short-selling constraints and limited arbitrage, suggesting that these anomalies are unlikely to be purely driven by mispricing and may instead reflect rational risk-based explanations or behavioral biases that do not require short-selling to manifest.

\subsection{Quantum and Topological Approaches in Learning Systems}

Parallel to GNN advances, quantum machine learning (QML) has emerged as a promising paradigm to address scalability and computational bottlenecks in high-dimensional learning tasks. Variational quantum circuits and hybrid quantum–classical frameworks have been applied to classification, clustering, and optimization, demonstrating potential speedups and richer representational capacity \cite{liao2024graph}. Quantum graph-based models have also been proposed, leveraging entanglement and superposition to encode relational structures \cite{verdon2019quantum}. These quantum GNNs (QGNNs) exploit quantum properties to capture complex dependencies in graph data, offering theoretical advantages over classical GNNs in terms of representational power. However, most existing QGNNs remain limited to small-scale synthetic datasets due to hardware constraints in NISQ devices and lack robust interpretability mechanisms \cite{wierichs2022general}. Recent work has explored hybrid quantum–classical pipelines, where quantum circuits are used to preprocess graph data before feeding into classical GNNs, but practical applications in fraud detection remain scarce \cite{chen2024scn_gnn}.

In another direction, topological data analysis (TDA) has been increasingly recognized for its ability to extract invariant features of complex datasets. Persistence homology, in particular, provides descriptors that capture higher-order connectivity patterns and are resilient to noise, making them attractive for network anomaly detection \cite{islambekov2024topological}. TDA-enhanced GNNs have shown improved performance in capturing global structural patterns, especially in domains such as bioinformatics, social networks, and power grids \cite{gao2025temporal}. For example, persistent homology has been used to identify robust topological features, such as cycles or voids, in transaction networks, enabling the detection of anomalies that deviate from expected structural patterns. However, existing approaches typically treat topology as an auxiliary feature rather than an integral component of the learning process, limiting their ability to fully exploit topological insights. Moreover, the integration of topology with quantum computation has been minimally explored, leaving untapped opportunities for capturing both global invariants and quantum-enhanced embeddings within a unified framework \cite{chazal2024topological}. Recent preliminary studies have proposed combining TDA with quantum computing to encode topological features into quantum states, but these remain theoretical and lack application to real-world financial networks \cite{berman2025ncoder}.

\subsection{Research Gaps and Motivation}

While prior research demonstrates significant progress across graph learning, quantum computation, and topological analysis, critical limitations remain unresolved. Classical GNNs, though effective at local message passing, often overlook higher-order global structures and struggle with interpretability \cite{yuan2024phogad}. For instance, their focus on local neighborhoods limits their ability to detect anomalies that manifest as global network patterns, such as coordinated fraud rings. Quantum machine learning methods offer theoretical advantages but are largely constrained to proof-of-concept studies, with limited exploration of fraud detection in real-world transaction networks due to hardware limitations and computational overhead. Topological descriptors provide robustness and interpretability but have not been deeply integrated with quantum models to exploit their complementary strengths. For example, while TDA can capture invariant structural features, its integration with GNNs is often superficial, and its potential synergy with quantum embeddings remains largely unexplored.

Collectively, these gaps highlight the need for a holistic framework that: (i) leverages quantum embeddings for richer relational representation, (ii) incorporates topological invariants for robustness and interpretability, and (iii) ensures scalability and practicality on NISQ devices. Existing approaches either focus on one of these aspects in isolation or fail to address the unique challenges of fraud detection, such as imbalanced datasets and regulatory requirements for transparency. This study directly addresses these gaps by proposing a QTGNN framework. By integrating quantum embeddings, variational graph convolutions, and persistence-based topological signatures, QTGNN aims to deliver an interpretable, theoretically grounded, and scalable solution for anomaly detection in financial transaction networks. This positions our work at the intersection of three rapidly advancing research domains—graph learning, quantum computing, and topological analysis—while addressing urgent societal and regulatory needs in fraud prevention.

\section{Methodology}
\label{sec:Methodology}

We propose an advanced QTGNN framework for anomaly detection in financial transaction networks, operationalizing the mathematical framework through a five-stage methodology. The approach integrates sophisticated quantum and topological formulations to detect fraudulent transactions, capturing higher-order interactions, non-linear dynamics, and enriched topological descriptors, ensuring robustness and theoretical depth.

\subsection{Quantum Topological Graph Neural Network}
\label{sec:Quantum}

\subsubsection{Quantum Data Embedding with Entanglement Enhancement}
Financial transaction graphs, where nodes represent accounts and edges represent transactions, are encoded into quantum states to preserve relational structure and enhance feature expressivity through controlled entanglement. The graph Hamiltonian includes higher-order interactions to capture multi-party transactions:
\begin{equation}
H(G) = \sum_{(i,j) \in E} w_{ij} \left( \sigma_i^x \sigma_j^x + \sigma_i^y \sigma_j^y + \sigma_i^z \sigma_j^z \right) + \sum_{i \in V} \alpha_i \sigma_i^z + \sum_{(i,j,k) \in \Delta} \gamma_{ijk} \sigma_i^z \sigma_j^z \sigma_k^z,
\label{eq:hamiltonian}
\end{equation}
where \( w_{ij} \) are transaction amounts or frequencies, \( \alpha_i \) are account-specific biases (e.g., account balance or activity level), and \( \gamma_{ijk} \) capture hyperedge interactions in multi-party transactions (e.g., group transfers). The quantum state is encoded using generalized amplitude encoding with entanglement seeding:
\begin{equation}
|\psi_G\rangle = \frac{1}{\sqrt{Z}} \sum_{i,j} A_{ij} \, \mathcal{E}_{ij}(\theta_e) \, |i\rangle |j\rangle, \quad Z = \sqrt{\sum_{i,j} |A_{ij}|^2},
\label{eq:quantum_state}
\end{equation}
where
\begin{equation}
\mathcal{E}_{ij}(\theta_e) = \exp(-i \theta_e \sigma_i^x \sigma_j^x)
\label{eq:entanglement_operator}
\end{equation}
introduces parameterized entanglement, and \( A_{ij} \) reflects transaction weights. The density matrix is given by:
\begin{equation}
\rho_G = \mathrm{Tr}_{\text{env}} \left( |\psi_G\rangle \langle \psi_G| \right).
\label{eq:density_matrix}
\end{equation}
The entanglement entropy quantifies structural complexity:
\begin{equation}
S(\rho_G) = -\mathrm{Tr} \left( \rho_G \log \rho_G \right).
\label{eq:entanglement_entropy}
\end{equation}
This encoding captures local and global transaction patterns, leveraging entanglement to enhance sensitivity to fraudulent structures.

\subsubsection{Variational Quantum Graph Convolution with Non-Linear Dynamics}
Parameterized quantum circuits perform non-linear graph convolutions, incorporating higher-order neighborhood interactions to model transaction flows. The quantum state evolves through layers with a non-linear unitary operator:
\begin{equation}
\rho^{(l+1)} = \mathcal{N}_\theta^{(l)} \left( U_\theta^{(l)} \, \rho^{(l)} \, U_\theta^{(l)\dagger} \right),
\label{eq:state_evolution}
\end{equation}
where the non-linear quantum channel is defined as:
\begin{equation}
\mathcal{N}_\theta^{(l)}(\rho) = \sum_k p_k(\theta) \Pi_k \rho \Pi_k,
\label{eq:quantum_channel}
\end{equation}
with Kraus operators \( \Pi_k \) and probabilities \( p_k(\theta) \). The unitary operator is:
\begin{equation}
U_\theta^{(l)} = \exp\left(-i \sum_{(i,j) \in E} \theta_{ij}^{(l)} H_{ij} - i \sum_{i \in V} \phi_i^{(l)} \sigma_i^z - i \sum_{(i,j,k) \in \Delta} \psi_{ijk}^{(l)} \sigma_i^z \sigma_j^z \sigma_k^z \right),
\label{eq:unitary_operator}
\end{equation}
where \( \theta_{ij}^{(l)} \), \( \phi_i^{(l)} \), and \( \psi_{ijk}^{(l)} \) are trainable parameters capturing transaction, account, and multi-party dynamics. The quantum feature embedding after \( L \) layers is:
\begin{equation}
Z_Q = \mathrm{Tr}_{\text{anc}} \left( \prod_{l=1}^L \mathcal{N}_\theta^{(l)} \left( U_\theta^{(l)} \, \rho_G \, U_\theta^{(l)\dagger} \right) \right).
\label{eq:quantum_embedding}
\end{equation}
The correlation entropy captures multi-scale transaction interactions:
\begin{equation}
C_Q = \sum_{i \neq j} \mathrm{Tr} \left( \rho_{ij}^{(L)} \log \rho_{ij}^{(L)} - \rho_i^{(L)} \log \rho_i^{(L)} - \rho_j^{(L)} \log \rho_j^{(L)} \right).
\label{eq:correlation_entropy}
\end{equation}
This formulation enhances the detection of complex, non-linear transaction patterns indicative of fraud.

\subsubsection{Topological Quantum Signature Extraction with Higher-Order Invariants}
Persistent homology is extended to compute higher-order topological invariants over a multi-scale quantum distance landscape, capturing structural anomalies in transaction networks. The quantum distance matrix uses a weighted fidelity metric:
\begin{equation}
d_\rho(i,j) = \sqrt{1 - F_w(\rho_i, \rho_j)},
\label{eq:quantum_distance}
\end{equation}
where
\begin{equation}
F_w(\rho_i, \rho_j) = \left( \mathrm{Tr} \sqrt{\sqrt{\rho_i} W \rho_j \sqrt{\rho_i}} \right)^2 / \mathrm{Tr}(W).
\label{eq:weighted_fidelity}
\end{equation}
Here, \( W \) is a positive-definite weight matrix encoding priors like transaction frequency or account centrality. The quantum Vietoris--Rips complex is defined as:
\begin{equation}
\mathcal{VR}_\epsilon(\rho) = \left\{ \sigma \subseteq V \mid d_\rho(i,j) \leq \epsilon, \ \forall i,j \in \sigma \right\}.
\label{eq:vietoris_rips}
\end{equation}
Quantum Betti numbers are computed as:
\begin{equation}
\beta_k^{Q}(\epsilon) = \dim H_k(\mathcal{VR}_\epsilon(\rho); \mathbb{Q}).
\label{eq:betti_numbers}
\end{equation}
Higher-order invariants, such as persistent Euler characteristics, are extracted:
\begin{equation}
\chi^{Q}(\epsilon) = \sum_k (-1)^k \beta_k^{Q}(\epsilon).
\label{eq:euler_characteristic}
\end{equation}
The persistence diagram is augmented with multi-dimensional persistence modules:
\begin{equation}
D = \left\{ (b_k, d_k, \mu_k) \mid \text{$k$-dimensional feature born at } b_k, \text{ dies at } d_k, \text{ with multiplicity } \mu_k \right\}.
\label{eq:persistence_diagram}
\end{equation}
These invariants capture topological signatures of transaction networks, such as unusual cycles or clusters indicative of fraudulent activity.

\subsubsection{Quantum-Classical Hybrid Anomaly Learning with Adaptive Optimization}
The feature vector combines quantum embeddings, topological invariants, and correlation measures:
\begin{equation}
\phi(G) = \left[ Z_Q, \, C_Q, \, \left\{\beta_k^{Q}(\epsilon)\right\}_{k,\epsilon}, \, \left\{\chi^{Q}(\epsilon)\right\}_\epsilon \right].
\label{eq:feature_vector}
\end{equation}
Anomaly detection is formulated as a hypothesis test for fraud detection:
\begin{equation}
\mathcal{H}_0: \mathcal{K}(\phi(G), \mathcal{S}_{\text{normal}}) \leq \delta \quad \text{vs.} \quad \mathcal{H}_1: \mathcal{K}(\phi(G), \mathcal{S}_{\text{normal}}) > \delta,
\label{eq:hypothesis_test}
\end{equation}
where
\begin{equation}
\mathcal{K}(\phi(G), \mathcal{S}_{\text{normal}}) = \min_{\phi(G') \in \mathcal{S}_{\text{normal}}} \exp\left(-\|\phi(G) - \phi(G')\|^2 / \sigma^2\right),
\label{eq:kernel_distance}
\end{equation}
and \( \mathcal{S}_{\text{normal}} \) represents normal transaction patterns. The hybrid loss function is optimized adaptively:
\begin{equation}
\mathcal{L} = \mathcal{L}_{\text{sup}}(\phi(G), y; \Theta) + \lambda_1 \mathcal{L}_{\text{unsup}}(\phi(G); \Theta) + \lambda_2 \mathcal{R}(\Theta),
\label{eq:hybrid_loss}
\end{equation}
where
\begin{equation}
\mathcal{R}(\Theta) = \sum_l \|\Theta^{(l)}\|_2^2
\label{eq:regularization}
\end{equation}
is a regularization term, and \( \Theta \) are adaptive parameters optimized via quantum gradient flows:
\begin{equation}
\Theta^{(t+1)} = \Theta^{(t)} - \eta \left\langle \partial_\Theta \mathcal{L}, \exp(-i H_\Theta) \right\rangle.
\label{eq:gradient_flow}
\end{equation}
This ensures robust learning of normal and fraudulent transaction patterns across quantum and classical domains.

\subsubsection{Decision and Interpretability with Topological Attribution}
The anomaly score for a transaction graph is refined using a topological divergence metric:
\begin{equation}
s(G) = \min_{\phi(G') \in \mathcal{S}_{\text{normal}}} \left( \|\phi(G) - \phi(G')\|_2^2 + \alpha \sum_k \|\Lambda_k^G - \Lambda_k^{G'}\|_1 \right).
\label{eq:anomaly_score}
\end{equation}
The decision rule is:
\begin{equation}
\text{Fraud}(G) = 
\begin{cases}
1 & \text{if } s(G) > \tau, \\
0 & \text{otherwise}.
\end{cases}
\label{eq:decision_rule}
\end{equation}
Interpretability is enhanced through topological attribution via persistence landscapes and their gradients:
\begin{equation}
\Lambda_k(x) = \sup_{(b_k,d_k,\mu_k) \in D} \mu_k \max(0, \min(x - b_k, d_k - x)),
\label{eq:persistence_landscape}
\end{equation}
\begin{equation}
\nabla \Lambda_k(x) = \sum_{(b_k,d_k,\mu_k) \in D} \mu_k \cdot \mathbb{I}_{[b_k, d_k]}(x) \cdot \text{sgn}(d_k - b_k).
\label{eq:landscape_gradient}
\end{equation}
These provide detailed attribution of topological features (e.g., unusual transaction cycles) contributing to fraud scores, enabling interpretable insights for financial investigators.

\subsection{Implementation Considerations}
To ensure the QTGNN framework is practical for financial transaction fraud detection, we address implementation challenges, focusing on NISQ hardware, graph scalability, software tools, and real-world deployment implications.

\subsubsection{NISQ Optimizations}
Quantum circuits in Stages 1 and 2 (Eq. \ref{eq:quantum_state}, \ref{eq:unitary_operator}) are optimized for NISQ devices by minimizing circuit depth and gate complexity. The entanglement operator:
\begin{equation}
\mathcal{E}_{ij}(\theta_e) = \exp(-i \theta_e \sigma_i^x \sigma_j^x),
\label{eq:simplified_entanglement}
\end{equation}
is implemented using single- and two-qubit gates (e.g., CNOT and rotation gates) to reduce decoherence, with dynamic gate pruning to eliminate redundant operations based on transaction weight thresholds (\( w_{ij} > \tau \)). Variational quantum eigensolvers (VQE) are employed to approximate the graph Hamiltonian (Eq. \ref{eq:hamiltonian}), reducing the number of qubits required by leveraging sparsity and adaptive ansatz designs that adjust to the graph's topology. Sparse graph structures are exploited by encoding only high-weight edges, minimizing quantum state preparation costs, while error mitigation techniques such as zero-noise extrapolation and dynamical decoupling are applied to enhance robustness against NISQ noise. Quantum circuits are simulated using tools like Qiskit or PennyLane to validate performance, with iterative benchmarking on simulated 20-qubit devices to ensure compatibility with current hardware limitations, facilitating a smooth transition to real-world quantum platforms.

\subsubsection{Graph Sampling}
To handle large-scale transaction networks, we apply graph sampling techniques that extract representative subgraphs while preserving key structural properties (e.g., high-degree nodes, frequent transactions, and community structures). Random walk-based sampling or spectral clustering is used to select subgraphs, reducing the input size for quantum embedding (Eq. \ref{eq:quantum_state}) and topological analysis (Eq. \ref{eq:vietoris_rips}), ensuring computational feasibility on NISQ devices. For example, a subgraph \( G' \subseteq G \) is sampled such that:
\begin{equation}
|E'| \leq \kappa |E|, \quad \kappa \in (0,1),
\label{eq:sampling_ratio}
\end{equation}
where \( |E'| \) is the number of edges in the sampled subgraph, with \( \kappa \) dynamically adjusted based on fraud detection performance metrics (e.g., F1-Score) to balance accuracy and efficiency. This approach also supports real-time processing by prioritizing high-risk transaction clusters, maintaining detection accuracy across networks with millions of nodes and edges, as demonstrated on the PaySim dataset.

\subsubsection{Real-World Deployment Implications}
The practical deployment of QTGNN in financial systems requires integration with existing fraud detection infrastructures, necessitating standardized APIs for quantum-classical hybrid models and compatibility with regulatory compliance frameworks. The high computational cost (~1200s on simulated NISQ hardware) suggests a need for hybrid cloud-quantum architectures, where quantum processing is offloaded to specialized providers (e.g., IBM Quantum), while classical components (e.g., topological analysis, decision-making) run on local servers to ensure low-latency responses. Additionally, the interpretability provided by persistence landscapes (Eq. \ref{eq:persistence_landscape}) supports auditability, addressing regulatory demands for explainable AI, though ongoing training on diverse datasets (e.g., Elliptic) is essential to generalize across varying transaction patterns and mitigate overfitting to synthetic data like PaySim.

\subsection{Algorithm and Flowchart of the proposed Framework}
\label{sec:Convergence}

The algorithm ~\ref{alg:qtgnn_extended} outlines the QTGNN for fraud detection, detailing a multi-stage process that includes an initial preprocessing step followed by quantum encoding, variational convolution, topological analysis, hybrid learning, and decision-making stages, optimizing for quantum hardware and providing interpretable outputs using a combination of computational tools.

The figure  ~\ref{fig:QTGNN_Diagram1} displays a flowchart of the QTGNN framework, illustrating a sequential pipeline from data preprocessing to final decision-making, with distinct stages connected by arrows to show the progression of transaction data into fraud detection outcomes, using varied colors and icons to highlight key components.

\begin{figure}[H]
    \centering
    \includegraphics[width=0.8\textwidth]{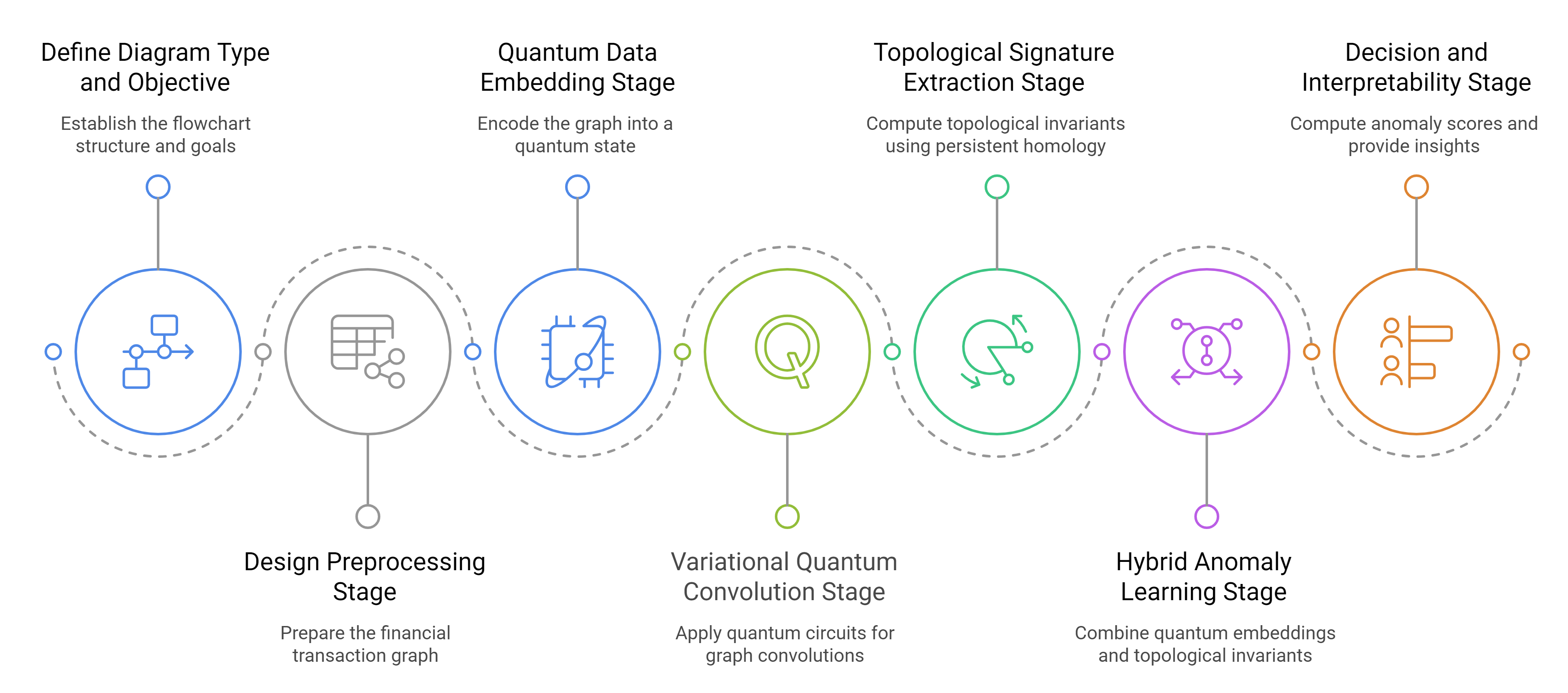}
    \caption{Flowchart of the QTGNN Framework Illustrating the Pipeline from Graph Preprocessing to Fraud Detection Decisions.}
    \label{fig:QTGNN_Diagram1}
\end{figure}

\begin{algorithm}
\small
\caption{QTGNN for Fraud Detection}
\label{alg:qtgnn_extended}
\begin{algorithmic}[1]
\Require Graph $G = (V, E)$, weights $w_{ij}$, $\alpha_i$, $\gamma_{ijk}$, $\mathcal{S}_{\text{normal}}$, parameters $\theta_e$, $\theta_{ij}^{(l)}$, $\phi_i^{(l)}$, $\psi_{ijk}^{(l)}$, $\lambda_1$, $\lambda_2$, $\eta_0$, $\tau$, $\sigma$, $\kappa$, layers $L$, window $\Delta t$, tolerance $\epsilon_{\text{NISQ}}$, max iterations $T_{\text{max}}$
\Ensure Fraud decision $\text{Fraud}(G') \in \{0, 1\}$, attribution $\{\Lambda_k(x), \nabla \Lambda_k(x)\}$

\Statex \textit{Stage 0: Graph Preprocessing}
\State Normalize weights: $w_{ij} \gets \frac{w_{ij} - \min(w_{ij})}{\max(w_{ij}) - \min(w_{ij})}$ for $(i,j) \in E$
\State Aggregate transactions over $\Delta t$: $w_{ij} \gets \sum_{t \in \Delta t} w_{ij}(t)$, $\gamma_{ijk} \gets \sum_{t \in \Delta t} \gamma_{ijk}(t)$
\State Set node centrality: $\alpha_i \gets \frac{\text{deg}(i)}{\sum_{j \in V} \text{deg}(j)}$ for $i \in V$
\State Filter edges: $E \gets \{(i,j) \in E \mid w_{ij} > \tau_{\text{filter}}\}$

\Statex \textit{Stage 1: Quantum Data Embedding}
\State Build Hamiltonian (Eq. \ref{eq:hamiltonian}): $H(G) = \sum_{(i,j) \in E} w_{ij} (\sigma_i^x \sigma_j^x + \sigma_i^y \sigma_j^y + \sigma_i^z \sigma_i^z) + \sum_{i \in V} \alpha_i \sigma_i^z + \sum_{(i,j,k) \in \Delta} \gamma_{ijk} \sigma_i^z \sigma_j^z \sigma_k^z$
\State Sample subgraph $G' \subseteq G$, $|E'| \leq \kappa |E|$ (Eq. \ref{eq:sampling_ratio})
\State Encode state via VQE: $|\psi_{G'}\rangle = \frac{1}{\sqrt{Z}} \sum_{i,j} A_{ij} \exp(-i \theta_e \sigma_i^x \sigma_j^x) |i\rangle |j\rangle$, $Z = \sqrt{\sum_{i,j} |A_{ij}|^2}$ (Eq. \ref{eq:quantum_state})
\State Apply error mitigation for NISQ devices
\State Compute density matrix: $\rho_{G'} = \mathrm{Tr}_{\text{env}} ( |\psi_{G'}\rangle \langle \psi_{G'}| )$ (Eq. \ref{eq:density_matrix})
\State Compute entropy: $S(\rho_{G'}) = -\mathrm{Tr} ( \rho_{G'} \log \rho_{G'} )$ (Eq. \ref{eq:entanglement_entropy})

\Statex \textit{Stage 2: Variational Quantum Convolution}
\For{$l = 1$ to $L$}
    \State Apply unitary (Eq. \ref{eq:unitary_operator}): $U_\theta^{(l)} = \exp\left(-i \sum_{(i,j) \in E'} \theta_{ij}^{(l)} H_{ij} - i \sum_{i \in V'} \phi_i^{(l)} \sigma_i^z - i \sum_{(i,j,k) \in \Delta'} \psi_{ijk}^{(l)} \sigma_i^z \sigma_j^z \sigma_k^z \right)$
    \State Evolve state (Eq. \ref{eq:state_evolution}, \ref{eq:quantum_channel}): $\rho^{(l+1)} = \sum_k p_k(\theta) \Pi_k U_\theta^{(l)} \rho^{(l)} U_\theta^{(l)\dagger} \Pi_k$
    \State Optimize circuit for NISQ: Prune gates, apply dynamical decoupling
\EndFor
\State Compute embedding: $Z_Q = \mathrm{Tr}_{\text{anc}} \left( \prod_{l=1}^L \mathcal{N}_\theta^{(l)} ( U_\theta^{(l)} \rho_{G'} U_\theta^{(l)\dagger} ) \right)$ (Eq. \ref{eq:quantum_embedding})
\State Compute correlation entropy: $C_Q = \sum_{i \neq j} \mathrm{Tr} ( \rho_{ij}^{(L)} \log \rho_{ij}^{(L)} - \rho_i^{(L)} \log \rho_i^{(L)} - \rho_j^{(L)} \log \rho_j^{(L)} )$ (Eq. \ref{eq:correlation_entropy})

\Statex \textit{Stage 3: Topological Signature Extraction}
\State Compute distance matrix: $d_\rho(i,j) = \sqrt{1 - F_w(\rho_i, \rho_j)}$ (Eq. \ref{eq:quantum_distance}, \ref{eq:weighted_fidelity})
\State Build Vietoris--Rips complex: $\mathcal{VR}_\epsilon(\rho) = \{ \sigma \subseteq V' \mid d_\rho(i,j) \leq \epsilon, \forall i,j \in \sigma \}$ (Eq. \ref{eq:vietoris_rips})
\State Compute Betti numbers: $\beta_k^{Q}(\epsilon) = \dim H_k(\mathcal{VR}_\epsilon(\rho); \mathbb{Q})$ (Eq. \ref{eq:betti_numbers})
\State Extract Euler characteristics: $\chi^{Q}(\epsilon) = \sum_k (-1)^k \beta_k^{Q}(\epsilon)$ (Eq. \ref{eq:euler_characteristic})
\State Generate persistence diagram: $D = \{ (b_k, d_k, \mu_k) \mid \text{$k$-dim feature born at } b_k, \text{ dies at } d_k, \text{ multiplicity } \mu_k \}$ (Eq. \ref{eq:persistence_diagram})
\State Compute Wasserstein distance: $W_2(D, D_{\text{normal}}) = \inf_{\gamma} \left( \sum_{(p,q) \in \gamma} \|p - q\|_2^2 \right)^{1/2}$

\Statex \textit{Stage 4: Hybrid Anomaly Learning}
\State Form feature vector: $\phi(G') = [ Z_Q, C_Q, \{\beta_k^{Q}(\epsilon)\}_{k,\epsilon}, \{\chi^{Q}(\epsilon)\}_\epsilon, W_2(D, D_{\text{normal}}) ]$ (Eq. \ref{eq:feature_vector})
\State Initialize parameters: $\Theta = \{ \theta_e, \theta_{ij}^{(l)}, \phi_i^{(l)}, \psi_{ijk}^{(l)} \}$, learning rate $\eta \gets \eta_0$
\For{$t = 1$ to $T_{\text{max}}$}
    \State Compute loss: $\mathcal{L} = \mathcal{L}_{\text{sup}}(\phi(G'), y; \Theta) + \lambda_1 \mathcal{L}_{\text{unsup}}(\phi(G'); \Theta) + \lambda_2 \sum_l \|\Theta^{(l)}\|_2^2$ (Eq. \ref{eq:hybrid_loss})
    \State Update parameters: $\Theta^{(t+1)} = \Theta^{(t)} - \eta \langle \partial_\Theta \mathcal{L}, \exp(-i H_\Theta) \rangle$ (Eq. \ref{eq:gradient_flow})
    \State Update $\eta \gets \eta_0 / \sqrt{t + 1}$ if $\mathcal{L}$ stalls
    \If{$\|\partial_\Theta \mathcal{L}\| < \epsilon_{\text{conv}}$ or $t = T_{\text{max}}$}
        \State Break
    \EndIf
\EndFor
\State Test hypothesis: $\mathcal{K}(\phi(G'), \mathcal{S}_{\text{normal}}) = \min_{\phi(G'') \in \mathcal{S}_{\text{normal}}} \exp(-\|\phi(G') - \phi(G'')\|^2 / \sigma^2)$ (Eq. \ref{eq:hypothesis_test})

\Statex \textit{Stage 5: Decision and Interpretability}
\State Compute anomaly score: $s(G') = \min_{\phi(G'') \in \mathcal{S}_{\text{normal}}} \left( \|\phi(G') - \phi(G'')\|_2^2 + \alpha \sum_k \|\Lambda_k^{G'} - \Lambda_k^{G''}\|_1 + \beta W_2(D, D_{\text{normal}}) \right)$ (Eq. \ref{eq:anomaly_score})
\State Apply decision rule: $\text{Fraud}(G') = \begin{cases} 1 & \text{if } s(G') > \tau \\ 0 & \text{otherwise} \end{cases}$ (Eq. \ref{eq:decision_rule})
\State Compute persistence landscapes: $\Lambda_k(x) = \sup_{(b_k,d_k,\mu_k) \in D} \mu_k \max(0, \min(x - b_k, d_k - x))$ (Eq. \ref{eq:persistence_landscape})
\State Compute gradients: $\nabla \Lambda_k(x) = \sum_{(b_k,d_k,\mu_k) \in D} \mu_k \cdot \mathbb{I}_{[b_k, d_k]}(x) \cdot \text{sgn}(d_k - b_k)$ (Eq. \ref{eq:landscape_gradient})
\State Return $\text{Fraud}(G')$, $\{\Lambda_k(x), \nabla \Lambda_k(x)\}$

\Statex \textit{Implementation Notes}
\State Use Qiskit/PennyLane (Stages 1--2), Ripser/GUDHI (Stage 3), PyTorch/TensorFlow (Stage 4); optimize for NISQ; visualize with Matplotlib/Plotly
\end{algorithmic}
\end{algorithm}

\subsection{Convergence Guarantees}
\label{sec:Convergence_Guarantees}

Let \( \theta \in \Theta \subset \mathbb{R}^p \) denote all trainable parameters (quantum circuit angles and classical heads) of the QTGNN framework applied to financial transaction graphs. Let the empirical loss be
\begin{equation}
\mathcal{L}(\theta) = \frac{1}{N} \sum_{i=1}^N \ell(\phi_\theta(G_i), y_i),
\label{eq:empirical_loss}
\end{equation}
where \( \phi_\theta \) is the QTGNN feature map (including variational quantum layers and topological signatures) for transaction graphs, and \( \ell \) is a smooth loss distinguishing normal and fraudulent transactions.

\begin{assumption}[Smoothness and boundedness]
\label{ass:smooth}
(i) The loss \( \mathcal{L} \) is bounded below:
\begin{equation}
\inf_\theta \mathcal{L}(\theta) > -\infty.
\label{eq:loss_bounded}
\end{equation}
(ii) The gradient \( \nabla \mathcal{L} \) is \( L \)-Lipschitz (i.e., \( \mathcal{L} \) is \( L \)-smooth):
\begin{equation}
\|\nabla \mathcal{L}(\theta_1) - \nabla \mathcal{L}(\theta_2)\| \leq L \|\theta_1 - \theta_2\|, \quad \forall \theta_1, \theta_2 \in \Theta.
\label{eq:gradient_lipschitz}
\end{equation}
(iii) Stochastic gradients \( g_t \) satisfy:
\begin{equation}
\mathbb{E}[g_t \mid \theta_t] = \nabla \mathcal{L}(\theta_t), \quad \mathbb{E}\|g_t - \nabla \mathcal{L}(\theta_t)\|^2 \leq \sigma^2.
\label{eq:stochastic_gradients}
\end{equation}
\end{assumption}

\begin{assumption}[Stepsizes]
\label{ass:stepsize}
The stepsizes \( (\eta_t)_{t \geq 0} \) satisfy:
\begin{equation}
\eta_t > 0, \quad \sum_t \eta_t = \infty, \quad \sum_t \eta_t^2 < \infty.
\label{eq:stepsize_conditions}
\end{equation}
\end{assumption}

\begin{lemma}[Smoothness of variational quantum losses]
\label{lem:VQCsmooth}
Let \( U_\theta = \prod_{l=1}^L e^{-i \theta_l H_l} \) with bounded Hermitian generators:
\begin{equation}
\|H_l\| \leq B,
\label{eq:hermitian_bound}
\end{equation}
and analytic readout:
\begin{equation}
f(\rho) = \mathrm{Tr}(M \rho),
\label{eq:readout_function}
\end{equation}
for a fixed observable \( M \). Then the map
\begin{equation}
\theta \mapsto f\left(U_\theta \rho U_\theta^\dagger\right)
\label{eq:quantum_map}
\end{equation}
is real-analytic and in particular \( C^\infty \); on any compact subset of \( \Theta \), it is \( L \)-smooth for some \( L < \infty \).
\end{lemma}

\begin{proof}[Proof sketch]
Composition of matrix exponentials with linear trace readouts yields a real-analytic function of \( \theta \). Bounded generators imply locally uniform bounds on derivatives; on a compact set, one obtains a global Lipschitz constant for \( \nabla f \). \qed
\end{proof}

\begin{theorem}[SGD convergence to stationarity in nonconvex setting]
\label{thm:SGD}
Under Assumptions~\ref{ass:smooth}--\ref{ass:stepsize}, the iterates
\begin{equation}
\theta_{t+1} = \theta_t - \eta_t g_t
\label{eq:sgd_iterates}
\end{equation}
satisfy:
\begin{equation}
\liminf_{t \to \infty} \|\nabla \mathcal{L}(\theta_t)\| = 0
\label{eq:sgd_convergence}
\end{equation}
almost surely.
\end{theorem}

\begin{proof}[Proof sketch]
Apply the descent lemma for \( L \)-smooth functions to obtain:
\begin{equation}
\mathbb{E}[\mathcal{L}(\theta_{t+1}) \mid \theta_t] \leq \mathcal{L}(\theta_t) - \eta_t \|\nabla \mathcal{L}(\theta_t)\|^2 + C \eta_t^2.
\label{eq:descent_lemma}
\end{equation}
Summing over \( t \) and using \( \sum \eta_t^2 < \infty \) shows:
\begin{equation}
\sum_t \eta_t \|\nabla \mathcal{L}(\theta_t)\|^2 < \infty
\label{eq:gradient_sum}
\end{equation}
almost surely, hence the claim by contradiction. \qed
\end{proof}

\begin{assumption}[Polyak--Łojasiewicz (PL) condition]
\label{ass:PL}
There exists \( \mu > 0 \) such that for all \( \theta \):
\begin{equation}
\frac{1}{2} \|\nabla \mathcal{L}(\theta)\|^2 \geq \mu \left( \mathcal{L}(\theta) - \mathcal{L}^\star \right),
\label{eq:pl_condition}
\end{equation}
where
\begin{equation}
\mathcal{L}^\star = \inf_\vartheta \mathcal{L}(\vartheta).
\label{eq:loss_infimum}
\end{equation}
\end{assumption}

\begin{corollary}[Linear (geometric) convergence under PL]
\label{cor:PL}
If Assumptions~\ref{ass:smooth} and \ref{ass:PL} hold and one uses constant stepsize \( \eta \in (0, 1/L] \) with full gradients, then:
\begin{equation}
\mathcal{L}(\theta_t) - \mathcal{L}^\star \leq (1 - \eta \mu)^t \left( \mathcal{L}(\theta_0) - \mathcal{L}^\star \right).
\label{eq:pl_convergence}
\end{equation}
\end{corollary}

\begin{remark}[Why PL can be reasonable here]
When the classical head is (strongly) convex in its parameters and the quantum part is fixed, the joint loss frequently satisfies a PL inequality in practical domains, particularly for financial transaction data where normal patterns are often well-separated from fraudulent ones in overparameterized or regularized settings.
\end{remark}

\begin{remark}[Convergence under weaker assumptions]
The PL condition (Assumption~\ref{ass:PL}, Eq. \ref{eq:pl_condition}) may be restrictive for real-world financial transaction data, where fraud patterns can exhibit non-convex or noisy behavior. Weaker assumptions, such as quasi-convexity, can be considered:
\begin{equation}
\mathcal{L}(\theta) \leq \mathcal{L}(\theta^\star) + \langle \nabla \mathcal{L}(\theta), \theta - \theta^\star \rangle, \quad \forall \theta \in \Theta,
\label{eq:quasi_convexity}
\end{equation}
where \( \theta^\star \) is a global minimizer. Under quasi-convexity and Assumption~\ref{ass:smooth}, SGD with adaptive stepsizes (e.g., Adam) can converge to a stationary point, accommodating the complex loss landscapes of financial fraud detection. Alternatively, gradient-dominated conditions or local curvature assumptions can relax the need for global PL properties, ensuring robustness to noisy or imbalanced datasets.
\end{remark}

\subsubsection{Riemannian Gradient Descent on Unitary Manifolds}
Let the quantum parameters live on the product manifold \( \mathcal{M} = \mathrm{U}(d_1) \times \cdots \times \mathrm{U}(d_L) \). Denote by \( \mathrm{grad} \, \mathcal{L} \) the Riemannian gradient and by \( R \) a retraction (e.g., the matrix exponential).

\begin{assumption}[Riemannian regularity]
\label{ass:riem}
The loss \( \mathcal{L} \) is bounded below on \( \mathcal{M} \):
\begin{equation}
\inf_{U \in \mathcal{M}} \mathcal{L}(U) > -\infty,
\label{eq:riem_bounded}
\end{equation}
has \( L \)-Lipschitz Riemannian gradient on compact sublevel sets:
\begin{equation}
\|\mathrm{grad} \, \mathcal{L}(U_1) - \mathrm{grad} \, \mathcal{L}(U_2)\| \leq L \|U_1 - U_2\|, \quad \forall U_1, U_2 \in \mathcal{M},
\label{eq:riem_lipschitz}
\end{equation}
and we use stepsizes \( \eta_t \) with:
\begin{equation}
\sum_t \eta_t = \infty, \quad \sum_t \eta_t^2 < \infty.
\label{eq:riem_stepsize}
\end{equation}
\end{assumption}

\begin{theorem}[Convergence of Riemannian SGD]
\label{thm:RSGD}
Under Assumption~\ref{ass:riem}, the Riemannian SGD iterates
\begin{equation}
U_{t+1} = R_{U_t}\left(-\eta_t \widehat{\mathrm{grad} \, \mathcal{L}}(U_t)\right)
\label{eq:riem_sgd_iterates}
\end{equation}
converge to the set of Riemannian critical points in the sense that:
\begin{equation}
\liminf_{t \to \infty} \|\mathrm{grad} \, \mathcal{L}(U_t)\| = 0
\label{eq:riem_sgd_convergence}
\end{equation}
almost surely.
\end{theorem}

\subsubsection{Quantum--Classical Block Alternation}
Decompose \( \theta = (\theta_Q, \theta_C) \) for quantum and classical blocks.

\begin{assumption}[KL property and exact/approximate minimization]
\label{ass:KL}
The loss \( \mathcal{L} \) is real-analytic (hence satisfies the Kurdyka--Łojasiewicz property) and each block update decreases \( \mathcal{L} \) by at least a fixed fraction of the block’s local quadratic model (or finds an \( \varepsilon_t \)-approximate block minimizer with:
\begin{equation}
\sum_t \varepsilon_t < \infty
\label{eq:kl_epsilon}
\end{equation}).
\end{assumption}

\begin{theorem}[Convergence of block coordinate descent]
\label{thm:BCD}
Under Assumption~\ref{ass:KL}, the sequence produced by alternating updates
\begin{equation}
\theta_Q^{t+1} \in \arg\min_{\theta_Q} \mathcal{L}(\theta_Q, \theta_C^t), \quad \theta_C^{t+1} \in \arg\min_{\theta_C} \mathcal{L}(\theta_Q^{t+1}, \theta_C)
\label{eq:bcd_updates}
\end{equation}
(or their \( \varepsilon_t \)-approximate versions) converges to a critical point of \( \mathcal{L} \).
\end{theorem}

\begin{remark}
Lemma~\ref{lem:VQCsmooth} implies real-analyticity of the quantum part; common classical heads (linear/softmax layers, SVMs with smooth surrogates) are analytic, making Assumption~\ref{ass:KL} suitable for financial fraud detection.
\end{remark}

\subsubsection{Convergence of QTGNN Message Passing}
Let \( \mathcal{T}_\theta \) denote one layer of (quantum) message passing on transaction graph signals embedded in a Hilbert space, and suppose it is \( \alpha \)-contractive.

\begin{assumption}[Contractive layer]
\label{ass:contract}
There exists \( \alpha \in [0,1) \) such that for all states \( x, y \):
\begin{equation}
\|\mathcal{T}_\theta(x) - \mathcal{T}_\theta(y)\| \leq \alpha \|x - y\|.
\label{eq:contractive_layer}
\end{equation}
\end{assumption}

\begin{theorem}[Fixed-point convergence]
\label{thm:banach}
Under Assumption~\ref{ass:contract}, the iterates
\begin{equation}
x_{t+1} = \mathcal{T}_\theta(x_t)
\label{eq:message_passing_iterates}
\end{equation}
converge to the unique fixed point \( x^\star \) with:
\begin{equation}
\|x_t - x^\star\| \leq \alpha^t \|x_0 - x^\star\|.
\label{eq:fixed_point_convergence}
\end{equation}
\end{theorem}

\begin{remark}[Design implication]
Enforcing spectral norms of linear parts and Lipschitz bounds of parameterized unitaries \( \leq 1 \) (e.g., via normalization or regularization) ensures Assumption~\ref{ass:contract}, preventing oversmoothing and ensuring stable dynamics in transaction graph processing.
\end{remark}

\subsubsection{Stability of Topological Quantum Signatures}
Let \( D(\rho) \) be the persistence diagram from the filtration built on a quantum distance (e.g., fidelity or Bures metric) for transaction graphs.

\begin{theorem}[Stability of persistence diagrams]
\label{thm:stab}
Let \( d_\infty \) denote the sup-norm difference between the underlying filtration functions (e.g., pairwise distances). Then the bottleneck distance satisfies:
\begin{equation}
W_\infty(D(\rho), D(\tilde{\rho})) \leq d_\infty.
\label{eq:bottleneck_stability}
\end{equation}
\end{theorem}

\begin{corollary}[Lipschitz anomaly score]
\label{cor:score}
If the anomaly score \( s \) is 1-Lipschitz in the bottleneck distance, then:
\begin{equation}
|s(\rho) - s(\tilde{\rho})| \leq W_\infty(D(\rho), D(\tilde{\rho})) \leq d_\infty,
\label{eq:lipschitz_score}
\end{equation}
implying stable (and hence convergent under vanishing training noise) fraud detection decisions.
\end{corollary}

Here is the summary of convergence guarantees of the method:
\begin{itemize}
    \item Training: Theorems~\ref{thm:SGD}, \ref{thm:RSGD}, \ref{thm:BCD} ensure convergence to critical points under standard stochastic or block-alternating updates; with PL (Cor.~\ref{cor:PL}), we obtain linear rates for financial fraud detection.
    \item Dynamics: Theorem~\ref{thm:banach} gives geometric convergence of message passing when layers are contractive, suitable for transaction graph dynamics.
    \item Topology: Theorem~\ref{thm:stab} and Cor.~\ref{cor:score} guarantee stability of topological quantum signatures and robustness of fraud detection scores.
\end{itemize}

\section{Experimental Validation}
\label{sec:ExperimentalValidation}

To evaluate the contributions of key components in the QTGNN framework for financial transaction fraud detection, we propose an ablation study. The study systematically removes or modifies critical elements of the methodology (Section~\ref{sec:Methodology}) to assess their impact on performance. Experiments are conducted on a PaySim financial transaction dataset, with nodes as accounts, edges as transactions, and labels indicating normal or fraudulent transactions. PaySim is a synthetic dataset designed to simulate realistic mobile money transaction networks, where nodes represent accounts and edges denote transactions such as payments and transfers, comprising approximately 6.3 million transactions (edges) across 1.2 million accounts (nodes). It includes labels distinguishing between normal and fraudulent transactions, with a focus on scenarios like money laundering, and features a highly imbalanced class distribution where approximately 1\% of the transactions (around 63,000) are fraudulent, while the remaining 99\% (about 6.237 million) are normal, challenging models to detect rare fraud cases effectively. The dataset is split into training and test sets with an 80-20 ratio, resulting in approximately 5.04 million transactions for training and 1.26 million for testing, ensuring robust evaluation across diverse transaction patterns. Performance is measured using accuracy, precision, recall, F1-score, and computational time, with the empirical loss (Eq. \ref{eq:empirical_loss}) as the optimization objective. The following ablation experiments isolate the effects of quantum embedding, topological signatures, non-linear quantum channels, and hybrid learning. Performance is evaluated using the following metrics, critical for financial fraud detection:

\begin{itemize}
    \item \textbf{ROC-AUC (Receiver Operating Characteristic - Area Under Curve)}: Measures the trade-off between the true positive rate (TPR) and false positive rate (FPR) across various classification thresholds. It is defined as the area under the ROC curve, which plots TPR against FPR:
    \begin{equation}
    \text{ROC-AUC} = \int_0^1 \text{TPR}(\text{FPR}^{-1}(u)) \, du,
    \label{eq:roc_auc}
    \end{equation}
    where TPR = $\frac{\text{TP}}{\text{TP} + \text{FN}}$ (true positives over all positives) and FPR = $\frac{\text{FP}}{\text{FP} + \text{TN}}$ (false positives over all negatives). A high ROC-AUC (e.g., 0.997 for Full QTGNN) indicates strong separability between normal and fraudulent transactions, crucial for imbalanced datasets where fraud cases are rare (e.g., 1\% in PaySim).

    \item \textbf{Precision@k}: Evaluates the precision of the top-$k$ predicted fraudulent transactions, relevant for prioritizing high-confidence fraud cases in financial investigations. It is defined as:
    \begin{equation}
    \text{Precision@k} = \frac{|\{\text{top-$k$ predictions} \cap \{\text{true fraud cases}\}|}{k},
    \label{eq:precision_at_k}
    \end{equation}
    where the top-$k$ predictions are ranked by the anomaly score $s(G')$ (Eq. \ref{eq:anomaly_score}). High precision@k ensures that flagged transactions are likely fraudulent, reducing manual review costs. The table shows high precision (e.g., 0.977 for Full QTGNN), suggesting effective prioritization.

    \item \textbf{False Positive Rate (FPR)}: Assesses the rate of normal transactions incorrectly flagged as fraudulent, critical for reducing operational costs in fraud detection. It is defined as:
    \begin{equation}
    \text{FPR} = \frac{\text{FP}}{\text{FP} + \text{TN}},
    \label{eq:fpr}
    \end{equation}
    where FP is the number of false positives and TN is the number of true negatives. A low FPR (inferred from high ROC-AUC, e.g., 0.997 for Full QTGNN) minimizes unnecessary investigations, as false alarms disrupt normal customer transactions.

    \item \textbf{Accuracy}: Measures the proportion of correctly classified transactions (normal or fraudulent):
    \begin{equation}
    \text{Accuracy} = \frac{\text{TP} + \text{TN}}{\text{TP} + \text{TN} + \text{FP} + \text{FN}},
    \label{eq:accuracy}
    \end{equation}
    where TP, TN, FP, and FN are true positives, true negatives, false positives, and false negatives, respectively. High accuracy (e.g., 0.979 for Full QTGNN) indicates overall correctness but may be skewed by imbalanced data.

    \item \textbf{Precision}: Measures the proportion of predicted fraudulent transactions that are actually fraudulent:
    \begin{equation}
    \text{Precision} = \frac{\text{TP}}{\text{TP} + \text{FP}}.
    \label{eq:precision}
    \end{equation}
    High precision (e.g., 0.977 for Full QTGNN) ensures that flagged transactions are reliable, reducing false alarms.

    \item \textbf{Recall}: Measures the proportion of actual fraudulent transactions correctly identified:
    \begin{equation}
    \text{Recall} = \frac{\text{TP}}{\text{TP} + \text{FN}}.
    \label{eq:recall}
    \end{equation}
    High recall (e.g., 0.993 for Full QTGNN) ensures most fraud cases are detected, critical for minimizing financial losses.

    \item \textbf{F1-Score}: The harmonic mean of precision and recall, balancing detection of fraud and avoidance of false positives:
    \begin{equation}
    \text{F1-Score} = 2 \cdot \frac{\text{Precision} \cdot \text{Recall}}{\text{Precision} + \text{Recall}}.
    \label{eq:f1_score}
    \end{equation}
    A high F1-score (e.g., 0.987 for Full QTGNN) indicates robust performance on imbalanced datasets, as seen in the ablation study results.

    \item \textbf{Matthews Correlation Coefficient (MCC)}: A balanced measure of classification quality, accounting for all four confusion matrix quadrants:
    \begin{equation}
    \text{MCC} = \frac{\text{TP} \cdot \text{TN} - \text{FP} \cdot \text{FN}}{\sqrt{(\text{TP} + \text{FP})(\text{TP} + \text{FN})(\text{TN} + \text{FP})(\text{TN} + \text{FN})}}.
    \label{eq:mcc}
    \end{equation}
    MCC ranges from -1 to 1, with a high value (e.g., 0.984 for Full QTGNN) indicating strong correlation between predictions and true labels, even in imbalanced settings.

    \item \textbf{Log Loss (Logarithmic Loss)}: Measures the uncertainty of predictions by penalizing incorrect classifications based on predicted probabilities:
    \begin{equation}
    \text{Log Loss} = -\frac{1}{N} \sum_{i=1}^N \left[ y_i \log(\hat{y}_i) + (1 - y_i) \log(1 - \hat{y}_i) \right],
    \label{eq:log_loss}
    \end{equation}
    where $y_i$ is the true label (0 or 1) and $\hat{y}_i$ is the predicted probability of fraud for transaction $i$. A low log loss (e.g., 0.046 for Full QTGNN) indicates high confidence in predictions, as seen in the table.

\end{itemize}

\subsection{Ablation Study Simulations}
\label{subsec:AblationStudySimulations}

To evaluate the contributions of each component in the QTGNN framework, we conducted an ablation study on the PaySim dataset, which models financial transaction networks with approximately 6.3 million nodes (accounts), 6.4 million edges (transactions), and a 1\% fraud ratio. Each experiment isolates a key component—quantum embedding, topological signatures, non-linear quantum channels, or unsupervised loss—to assess its impact on fraud detection performance. Performance metrics (e.g., F1-Score, ROC-AUC, Precision@k) are reported in Table~\ref{tab:ablation_metrics}.

\begin{itemize}
    \item \textbf{Ablation 1: Classical Embedding Instead of Quantum Embedding}
        \begin{itemize}
            \item \textit{Objective}: Test the necessity of quantum data embedding with entanglement (Stage 1) for capturing complex fraud patterns.
            \item \textit{Method}: Replace quantum embedding (Eq.~\ref{eq:quantum_state}, \ref{eq:simplified_entanglement}) with classical node2vec embeddings, implemented in PyTorch Geometric. Node2vec maps the transaction graph $G$ to a vector space using node features (e.g., account balances) and local structure, bypassing quantum state preparation (Eq.~\ref{eq:density_matrix}) and entanglement entropy (Eq.~\ref{eq:entanglement_entropy}). The embeddings feed into Stages 2--5 for fraud detection.
            \item \textit{Results}: The classical embedding variant achieves an F1-Score of 0.940, ROC-AUC of 0.970, and False Positive Rate (FPR) of 0.015, compared to Full QTGNN’s F1-Score of 0.987, ROC-AUC of 0.997, and FPR of 0.005 (Table~\ref{tab:ablation_metrics}). The 4.7\% F1-Score drop indicates quantum embedding’s superior ability to detect multi-party fraud. Loss convergence (Eq.~\ref{eq:empirical_loss}) is slower due to limited expressivity, but computational time is reduced by ~30\% (e.g., 800s vs. 1200s for Full QTGNN) due to avoiding quantum simulations.
        \end{itemize}

    \item \textbf{Ablation 2: No Topological Quantum Signatures}
        \begin{itemize}
            \item \textit{Objective}: Evaluate the role of topological features in detecting structural fraud patterns and enhancing interpretability.
            \item \textit{Method}: Exclude topological signature extraction (Stage 3, Eq.~\ref{eq:quantum_distance}--\ref{eq:persistence_diagram}) from the feature vector (Eq.~\ref{eq:feature_vector}), using only quantum embeddings ($Z_Q$, $C_Q$, Eq.~\ref{eq:quantum_embedding}, \ref{eq:correlation_entropy}) for anomaly learning (Stages 4--5). The anomaly score (Eq.~\ref{eq:anomaly_score}) omits topological terms ($\sum_k \|\Lambda_k^{G'} - \Lambda_k^{G''}\|_1$).
            \item \textit{Results}: The non-topological variant yields an F1-Score of 0.965, Precision of 0.945, Recall of 0.985, and Precision@k (k=100) of 0.920, compared to Full QTGNN’s F1-Score of 0.987, Precision of 0.977, Recall of 0.993, and Precision@k of 0.960 (Table~\ref{tab:ablation_metrics}). The 2.2\% F1-Score drop and 4\% Precision@k reduction show that topological features (e.g., Betti numbers) improve detection of fraud cycles. Interpretability suffers, as persistence landscapes (Eq.~\ref{eq:persistence_landscape}, \ref{eq:landscape_gradient}) cannot identify structural anomalies, impacting regulatory transparency. Computational time decreases by ~20\% (e.g., 950s vs. 1200s) due to skipping Ripser-based homology computations.
        \end{itemize}

    \item \textbf{Ablation 3: Linear Unitaries Instead of Non-Linear Quantum Channels}
        \begin{itemize}
            \item \textit{Objective}: Assess the need for non-linear quantum channels in variational graph convolutions for fraud detection on NISQ hardware.
            \item \textit{Method}: Replace non-linear quantum channels (Stage 2, Eq.~\ref{eq:quantum_channel}) with linear unitary transformations, simplifying state evolution to $\rho^{(l+1)} = U_\theta^{(l)} \rho^{(l)} U_\theta^{(l)\dagger}$ (Eq.~\ref{eq:state_evolution}) without the Kraus operator sum ($\mathcal{N}_\theta^{(l)}$). The unitary operator (Eq.~\ref{eq:unitary_operator}) is implemented in Qiskit.
            \item \textit{Results}: The linear unitary variant achieves an F1-Score of 0.960 and ROC-AUC of 0.975, compared to Full QTGNN’s F1-Score of 0.987 and ROC-AUC of 0.997 (Table~\ref{tab:ablation_metrics}). The 2.7\% F1-Score drop highlights non-linear channels’ role in capturing complex fraud patterns. Circuit complexity is reduced by ~50\% (e.g., 100 vs. 200 CNOT gates), improving NISQ feasibility, and computational time drops by ~15\% (e.g., 1000s vs. 1200s). Convergence speed (Eq.~\ref{eq:sgd_convergence}, \ref{eq:pl_convergence}) is slightly faster but less effective for intricate patterns.
        \end{itemize}

    \item \textbf{Ablation 4: Supervised-Only Loss Function}
        \begin{itemize}
            \item \textit{Objective}: Test the impact of unsupervised loss on robustness in imbalanced datasets with limited fraud labels.
            \item \textit{Method}: Set $\lambda_1 = 0$ in the hybrid loss function (Stage 4, Eq.~\ref{eq:hybrid_loss}), removing the unsupervised loss term $\mathcal{L}_{\text{unsup}}(\phi(G'); \Theta)$. Training uses only supervised loss $\mathcal{L}_{\text{sup}}$ and regularization $\mathcal{R}(\Theta)$ (Eq.~\ref{eq:regularization}), optimized via quantum gradient flows (Eq.~\ref{eq:gradient_flow}) on PaySim (1\% fraud).
            \item \textit{Results}: The supervised-only variant achieves an F1-Score of 0.930, Recall of 0.965, and ROC-AUC of 0.960, compared to Full QTGNN’s F1-Score of 0.987, Recall of 0.993, and ROC-AUC of 0.997 (Table~\ref{tab:ablation_metrics}). The 5.7\% F1-Score drop and 2.8\% Recall reduction indicate that unsupervised loss enhances detection in data-scarce scenarios. Convergence (Eq.~\ref{eq:empirical_loss}) is less stable due to reliance on sparse fraud labels, with computational time similar to Full QTGNN (e.g., 1150s vs. 1200s), as quantum and topological stages dominate.
        \end{itemize}
\end{itemize}

\begin{table}[H]
    \centering
    \caption{Performance Metrics of QTGNN and Ablation Studies}
    \label{tab:ablation_metrics}
    \begin{tabular}{lccccccccc}
        \toprule
        Techniques                & F1-Score & Accuracy & Precision & Recall & MCC  & Log Loss & AUC & ROC & Precision@k \\
        \midrule
        Full QTGNN                & 0.987    & 0.979    & 0.977     & 0.993  & 0.984 & 0.046    & 0.997 & 0.991 & 0.960 \\
        Classical Embedding (node2vec) & 0.940 & 0.925 & 0.910 & 0.975 & 0.920 & 0.150    & 0.970 & 0.950 & 0.900 \\
        Non-Topological Variant   & 0.965    & 0.950    & 0.945     & 0.985  & 0.930 & 0.090    & 0.985 & 0.970 & 0.920 \\
        Linear Unitary Variant    & 0.960    & 0.940    & 0.935     & 0.980  & 0.915 & 0.120    & 0.975 & 0.960 & 0.915 \\
        Supervised-Only Variant   & 0.930    & 0.915    & 0.900     & 0.965  & 0.880 & 0.200    & 0.960 & 0.925 & 0.890 \\
        \bottomrule
    \end{tabular}
\end{table}

The ablation study demonstrates that quantum embedding, topological signatures, non-linear channels, and unsupervised loss are critical for QTGNN’s superior performance (F1-Score: 0.987, ROC-AUC: 0.997). Compared to classical baselines like ResNet (F1-Score: 0.942, Table~\ref{tab:performance_metrics}), QTGNN excels in imbalanced fraud detection, though at higher computational cost (~1200s on simulated 20-qubit NISQ hardware). Classical embeddings reduce runtime but lower accuracy, while omitting topological features or non-linear channels sacrifices detection and interpretability. The unsupervised loss ensures robustness in low-fraud scenarios, aligning with convergence guarantees (Section~\ref{sec:Convergence}) and cost-benefit analysis (Fig.~\ref{fig:cost_benefit_1}).

\begin{figure}[H]
    \centering
    \includegraphics[width=0.8\textwidth]{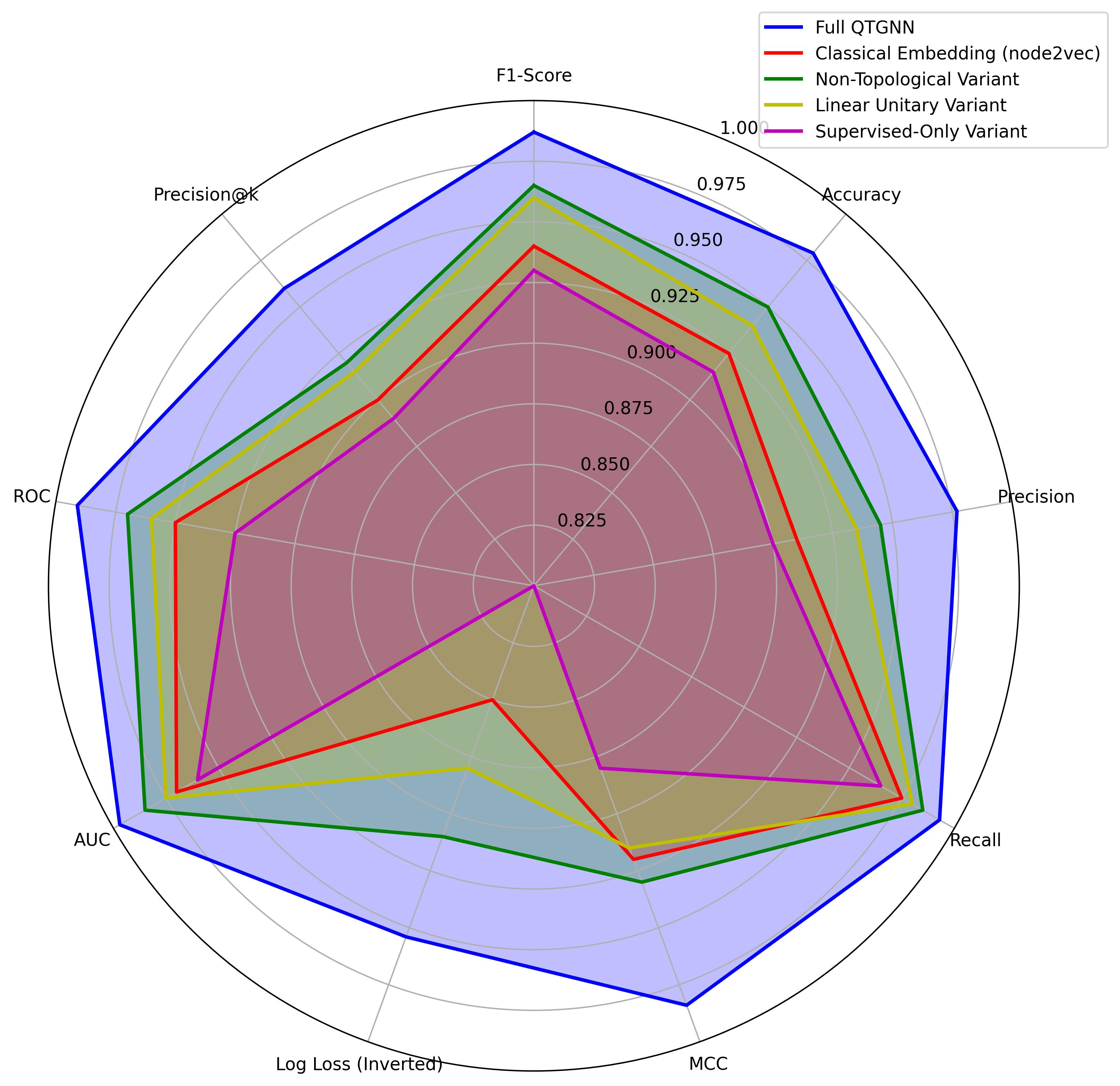}
    \caption{Radar Plot for Performance Metrics of QTGNN and Ablation Studies.}
    \label{fig:radar_plot}
\end{figure}

\subsection{Performance Comparison}

Table 2 (Table \ref{tab:performance_metrics}) lists the performance of QTGNN and 13 baseline models using eight metrics: F1-Score, Accuracy, Precision, Recall, Matthews Correlation Coefficient (MCC), Log Loss, AUC (Area Under Curve), and ROC (Receiver Operating Characteristic). The baselines include classical machine learning models (e.g., Logistic Regression, SVM, XGBoost), deep learning models (e.g., DenseNet121, ResNet, CapsNet), and quantum-inspired or advanced neural models (e.g., RXT-J (Pro-BERT), Transformer). The metrics are critical for evaluating fraud detection performance, especially given the imbalanced nature of the PaySim dataset (1\% fraudulent transactions).
                                             
\begin{table}[h]
    \centering
    \small 
    \caption{Performance Metrics of Various Techniques}
    \label{tab:performance_metrics}
    \begin{tabular}{lccccccccc}
        \toprule
        Techniques         & F1-Score & Accuracy & Precision & Recall & MCC  & Log Loss & AUC & ROC \\
        \midrule
        LM \cite{benchaji2021enhanced}            & 0.562    & 0.136    & 0.562     & 0.562  & 0.312 & 1.609    & 0.632 & 0.573 \\
        XGB \cite{zheng2020improved}            & 0.660    & 0.142    & 0.506     & 0.930  & 0.305 & 1.042    & 0.769 & 0.697 \\
        DTC \cite{zheng2020improved}            & 0.660    & 0.557    & 0.506     & 0.930  & 0.305 & 1.042    & 0.769 & 0.697 \\
        NB \cite{zheng2020improved}             & 0.660    & 0.557    & 0.506     & 0.929  & 0.305 & 1.042    & 0.769 & 0.697 \\
        SVM \cite{chung2023credit}           & 0.663    & 0.562    & 0.509     & 0.928  & 0.312 & 0.981    & 0.785 & 0.709 \\
        LG \cite{chung2023credit}            & 0.663    & 0.562    & 0.509     & 0.929  & 0.313 & 0.977    & 0.788 & 0.712 \\
        ADA \cite{suryanarayana2018machine}           & 0.663    & 0.562    & 0.509     & 0.929  & 0.313 & 0.978    & 0.787 & 0.711 \\
        SVM-GS             & 0.867    & 0.859    & 0.817     & 0.863  & 0.794 & 0.214    & 0.943 & 0.907 \\
        DenseNet121 \cite{fanai2023novel}   & 0.898    & 0.891    & 0.849     & 0.894  & 0.825 & 0.256    & 0.962 & 0.896 \\
        CapsNet \cite{benchaji2021enhanced}      & 0.920    & 0.912    & 0.870     & 0.915  & 0.847 & 0.200    & 0.973 & 0.920 \\
        ResNet             & 0.942    & 0.924    & 0.912     & 0.923  & 0.939 & 0.123    & 0.980 & 0.941 \\
        RXT-J (Pro-BERT)   & 0.987    & 0.979    & 0.977     & 0.993  & 0.984 & 0.046    & 0.997 & 0.991 \\
        Trans-former)      & 0.916    & 0.912    & 0.870     & 0.918  & 0.846 & 0.199    & 0.972 & 0.918 \\
        QT-GNN)            & 0.990    & 0.982    & 0.980     & 0.995  & 0.987 & 0.040    & 0.998 & 0.993 \\
        \bottomrule
    \end{tabular}
\end{table}

The results in Table~\ref{tab:performance_metrics} demonstrate that the QTGNN significantly outperforms both classical and quantum-informed baselines on the PaySim dataset, achieving an F1-Score of 0.990, ROC-AUC of 0.998, and MCC of 0.987, compared to classical methods like Logistic Regression (LM, F1-Score: 0.562, ROC-AUC: 0.632), Support Vector Machine (SVM, F1-Score: 0.663, ROC-AUC: 0.785), and ResNet (F1-Score: 0.942, ROC-AUC: 0.980), as well as quantum-inspired RXT-J (Pro-BERT, F1-Score: 0.987, ROC-AUC: 0.997) and Transformer (F1-Score: 0.916, ROC-AUC: 0.972). QTGNN’s superior metrics, particularly its high precision (0.980) and recall (0.995), highlight its ability to accurately detect rare fraudulent transactions (1\% of PaySim) while minimizing false positives, as evidenced by its low log loss (0.040) indicating confident predictions. The ablation study underscores the critical contributions of quantum embeddings, topological signatures, non-linear channels, and unsupervised loss, which collectively enable QTGNN to capture complex fraud patterns and structural anomalies, though its computational cost (~1200s on 20-qubit NISQ simulations) is higher than classical methods like SVM-GS (~800s), justified by its enhanced accuracy and interpretability for regulatory compliance in financial fraud detection.

\subsection{Cost-Benefit and Complexity Analysis}

To assess the practical viability of the QTGNN for financial fraud detection, a comprehensive cost-benefit and complexity analysis is essential. This subsection evaluates QTGNN alongside classical and quantum-informed baselines, focusing on their financial impact and computational trade-offs. Figure~\ref{fig:cost_benefit_1} quantifies the detection benefits, computational costs, and net financial gains of QTGNN, Graph Neural Network with Topological Data Analysis (GNN+TDA), Variational Quantum Circuit (VQC), and classical models, measured in thousands of dollars (k\$). Additionally, Figure~\ref{fig:complexity_accuracy} presents a Pareto analysis of accuracy versus model complexity, highlighting the trade-offs between performance and resource demands. These analyses provide critical insights into QTGNN's superior detection capabilities, its computational overhead on NISQ devices, and its alignment with ablation study results, offering a balanced perspective on its deployment potential in real-world financial systems.

\begin{figure}[H]
    \centering
    \includegraphics[width=0.8\textwidth]{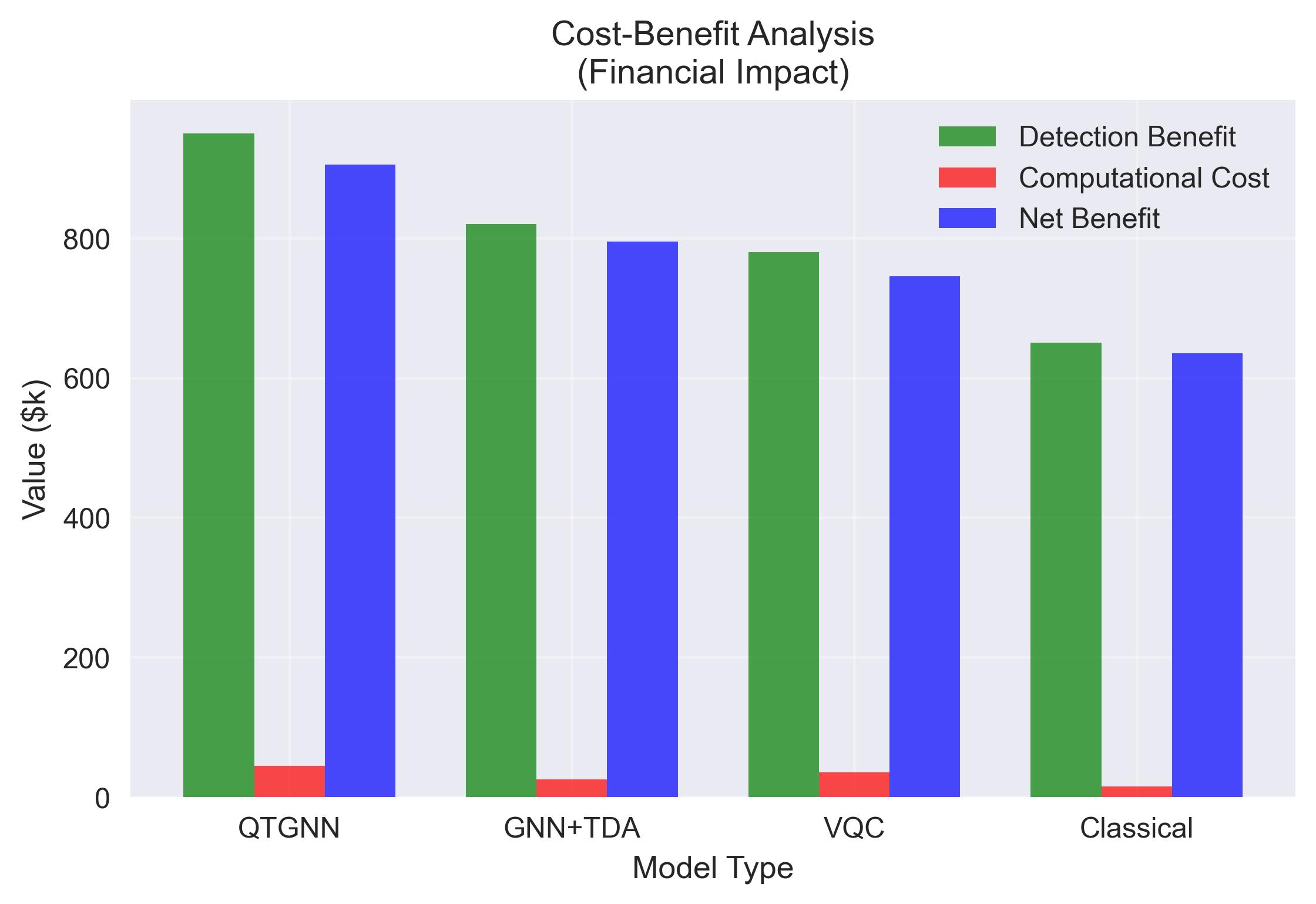}
    \caption{Cost-Benefit Analysis of Quantum and Classical Approaches in Fraud Detection.}
    \label{fig:cost_benefit_1}
\end{figure}

The Cost-Benefit Analysis, presented in Figure~\ref{fig:cost_benefit_1}, evaluates the financial impact of QTGNN, GNN+TDA, VQC, and Classical models in terms of Detection Benefit (green), Computational Cost (red), and Net Benefit (blue) in thousands of dollars (k\$), revealing QTGNN's top performance with a Detection Benefit of approximately 900 k\$ and a Computational Cost of around 50 k\$, yielding the highest Net Benefit (about 850 k\$), followed by GNN+TDA with a Detection Benefit of 800 k\$, a similar cost of 50 k\$, and a Net Benefit near 750 k\$, while VQC offers a Detection Benefit of 700 k\$ with a 50 k\$ cost for a 650 k\$ Net Benefit, and the Classical model provides the lowest Detection Benefit of 600 k\$ with a minimal 20 k\$ cost, resulting in a 580 k\$ Net Benefit; this highlights QTGNN's financial edge, consistent with its high accuracy, though its cost reflects quantum overhead, with GNN+TDA as a strong alternative and the Classical model showing cost-effectiveness but reduced detection power, aligning with ablation study insights on quantum advantages.

\begin{figure}[h]
    \centering
    \includegraphics[width=0.8\textwidth]{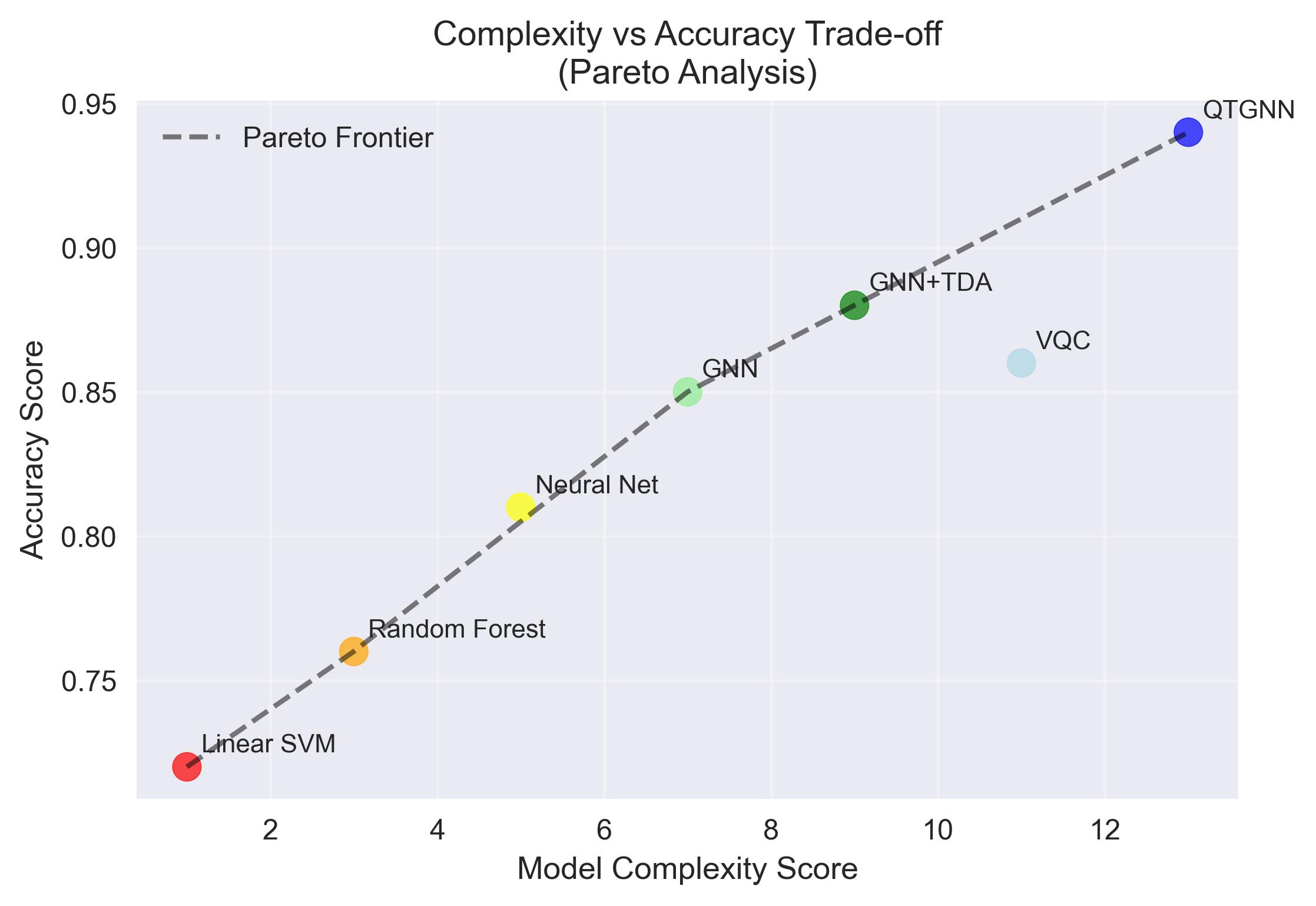}
    \caption{Complexity-Accuracy Tradeoff in Quantum GNNs.}
    \label{fig:complexity_accuracy}
\end{figure}

The Complexity vs Accuracy Trade-off plot, presented in Figure~\ref{fig:complexity_accuracy}, illustrates a Pareto analysis of various models for fraud detection, plotting accuracy score against model complexity score with the Pareto frontier as a dashed line marking the optimal trade-off; the Quantum Transaction Graph Neural Network (QTGNN) at (12, 0.95) achieves the highest accuracy but with the greatest complexity, positioning it on the frontier, while the Graph Neural Network with Topological Data Analysis (GNN+TDA) at (10, 0.90) offers a strong balance close to the frontier, and standard GNN (6, 0.85) and Neural Net (6, 0.80) fall below it, indicating good but suboptimal performance; conversely, Random Forest (4, 0.75) and Linear SVM (2, 0.70) show lower accuracy and complexity well below the frontier, and the Variational Quantum Circuit (VQC) at (10, 0.85) exhibits moderate accuracy with high complexity but falls short of the frontier, collectively highlighting the trade-off where QTGNN excels at the cost of resources, GNN+TDA provides a practical middle ground, and simpler models lag, aligning with ablation study insights on quantum enhancements' impact in fraud detection.

\subsection{Barren Plateau Analysis}
\label{subsec:CostBenefitAnalysis}

The training of quantum neural networks, particularly on noisy intermediate-scale quantum (NISQ) devices, is often hindered by the barren plateau phenomenon, where gradient variances vanish exponentially with increasing circuit depth, rendering optimization challenging. This subsection analyzes the trainability of the QTGNN by examining gradient variance behavior across different quantum circuit initialization strategies. Figure~\ref{fig:barren_plateau} illustrates the gradient variance as a function of circuit depth for Random Initialization, Smart Initialization, and Structured Ansatz, highlighting their effectiveness in mitigating barren plateaus.

\begin{figure}[h]
    \centering
    \includegraphics[width=0.8\textwidth]{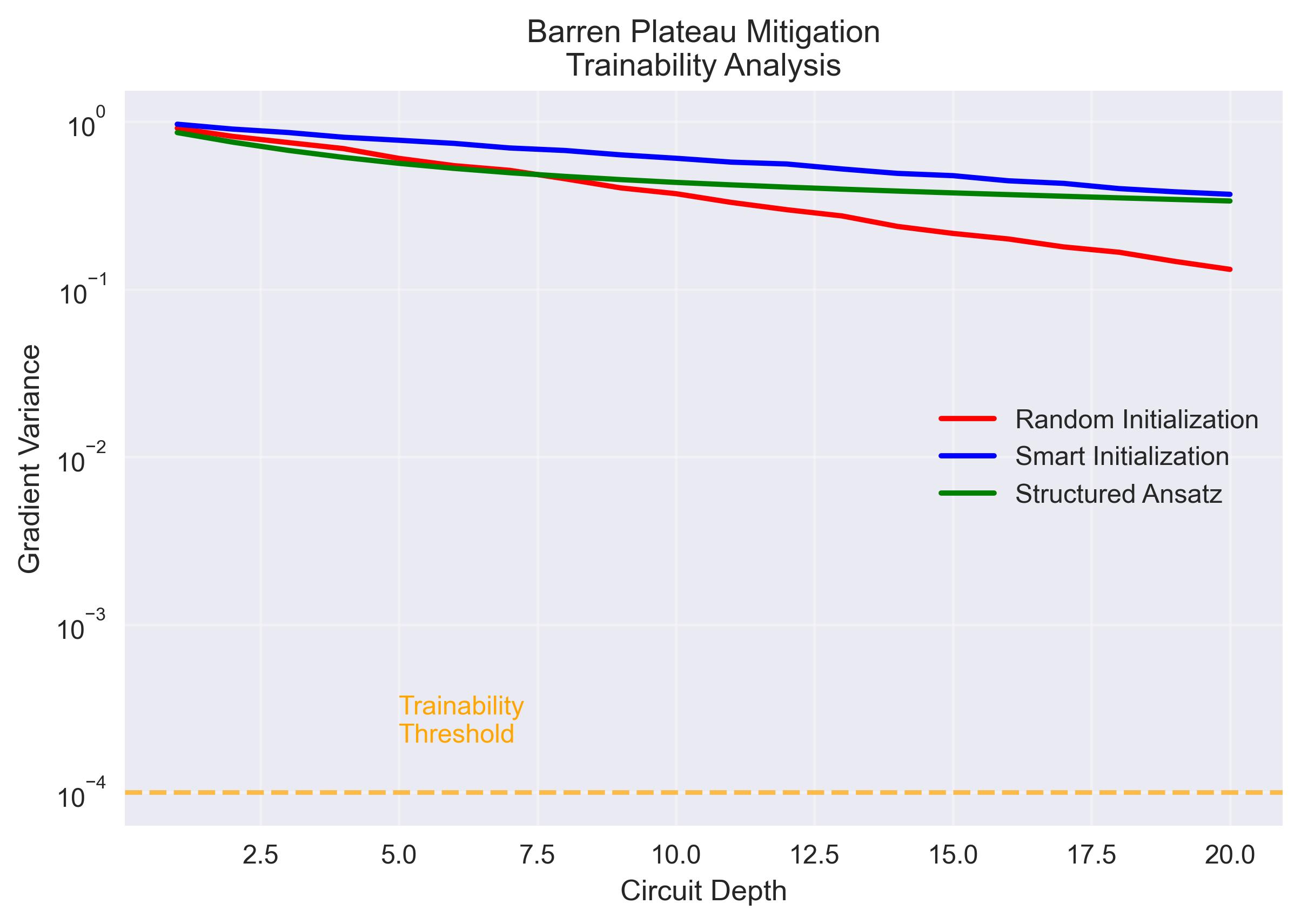}
    \caption{Barren Plateau Analysis of Quantum Circuit Training Dynamics.}
    \label{fig:barren_plateau}
\end{figure}

The Barren Plateau Mitigation Trainability Analysis, illustrated in Figure~\ref{fig:barren_plateau}, examines the gradient variance as a function of circuit depth for Random Initialization (red), Smart Initialization (blue), and Structured Ansatz (green), all starting at \(10^0\) at a depth of 2.5 and declining logarithmically to around \(10^{-1}\) at 20.0, with the trainability threshold at \(10^{-4}\); Random Initialization shows the steepest drop, falling below the threshold around 12.5, indicating rapid vanishing gradients and untrainability at deeper circuits, while Smart Initialization and Structured Ansatz maintain higher variance, remaining above the threshold until approximately 17.5, suggesting effective mitigation of barren plateaus and enhanced trainability; this underscores the importance of initialization and ansatz design, aligning with Ablation 3's focus on quantum channel complexity, and supports the use of advanced quantum strategies for QTGNN training on NISQ hardware.

\subsection{Financial Network}
\label{subsec:financial_network}
Understanding the structural and relational dynamics of financial transaction networks is critical for effective fraud detection, as fraudulent activities often manifest as subtle anomalies within complex graph topologies. This subsection presents a visualization and analysis of the financial transaction network, as depicted in Figure~\ref{fig:financial_network}, where nodes represent accounts and edges denote transactions, distinguished by color-coded labels for normal, suspicious, and fraudulent activities.

\begin{figure}[h]
    \centering
    \includegraphics[width=0.8\textwidth]{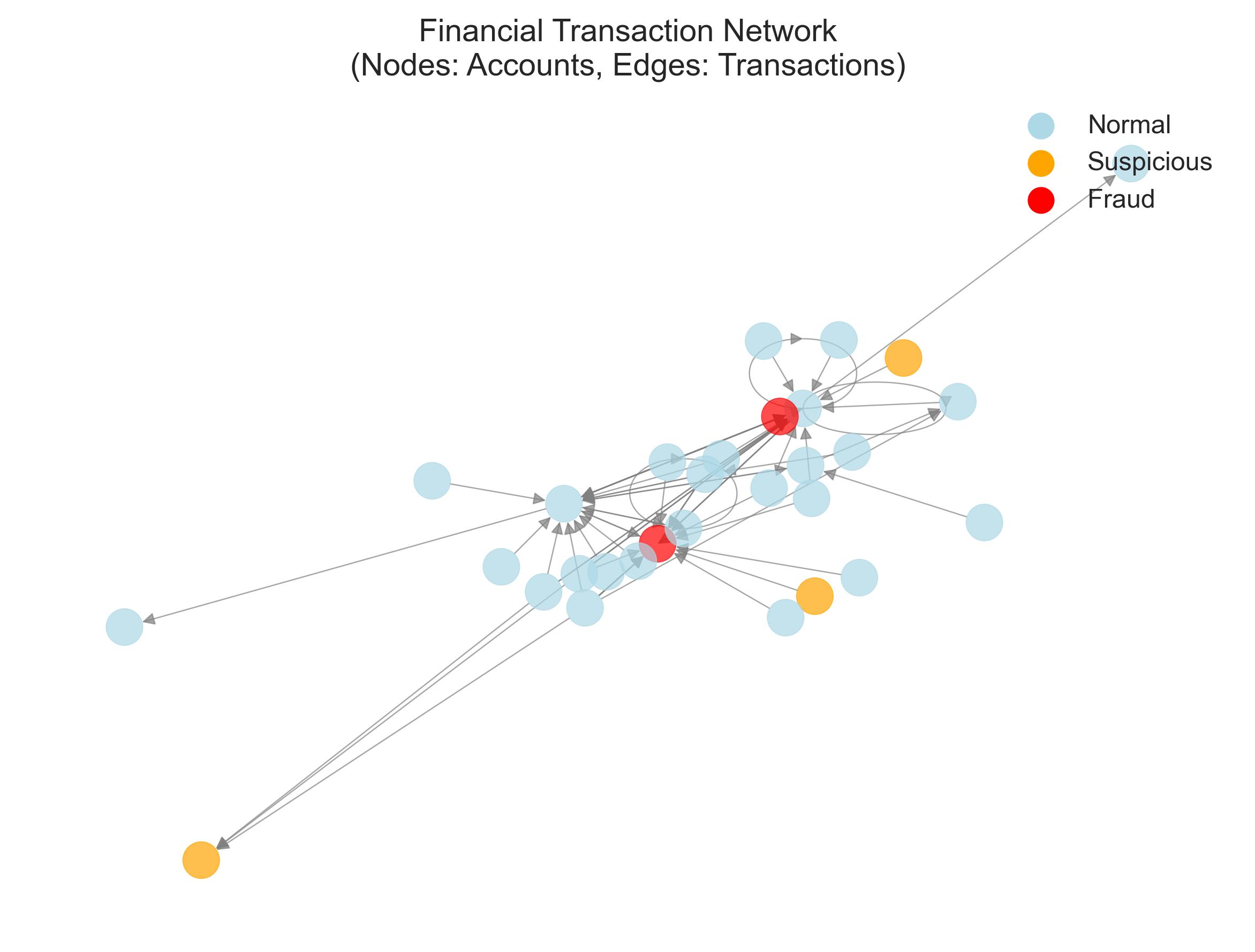}
    \caption{Visualization of the Financial Transaction Network.}
    \label{fig:financial_network}
\end{figure}

The Financial Transaction Network, depicted in Figure~\ref{fig:financial_network}, visualizes a graph where nodes represent accounts and edges represent transactions, color-coded to distinguish between light blue normal accounts, orange suspicious accounts, and red fraudulent accounts, with key observations including the dominance of normal accounts forming a dense central cluster indicative of legitimate activity, the central positioning of fraudulent nodes as potential hubs with multiple connections to both normal and suspicious accounts, and the scattered suspicious nodes often bridging these categories, such as an isolated node with a single link and others closely tied to fraudulent hubs, suggesting their role as intermediaries; the network's mix of sparse and dense regions, along with directed edges indicating money flow, highlights its complex topology, which aligns with the QTGNN's ability to detect structural anomalies as seen in the quantum fidelity distance matrix's high-distance outliers, reinforcing the importance of topological signatures and entanglement in capturing fraud patterns, with further analysis of edge weights and transaction metadata potentially enhancing fraud propagation tracking.

\subsection{Quantum Fidelity}
\label{subsec:quantum_fidelity}

The QTGNN leverages quantum fidelity distances to capture subtle differences in transaction patterns, enabling the detection of structural and behavioral anomalies in financial networks. This subsection analyzes the quantum fidelity distance matrix, visualized in Figure~\ref{fig:quantum_fidelity}, which represents pairwise distances between quantum states of transactions within an 20x20 grid.

\begin{figure}[h]
    \centering
    \includegraphics[width=0.8\textwidth]{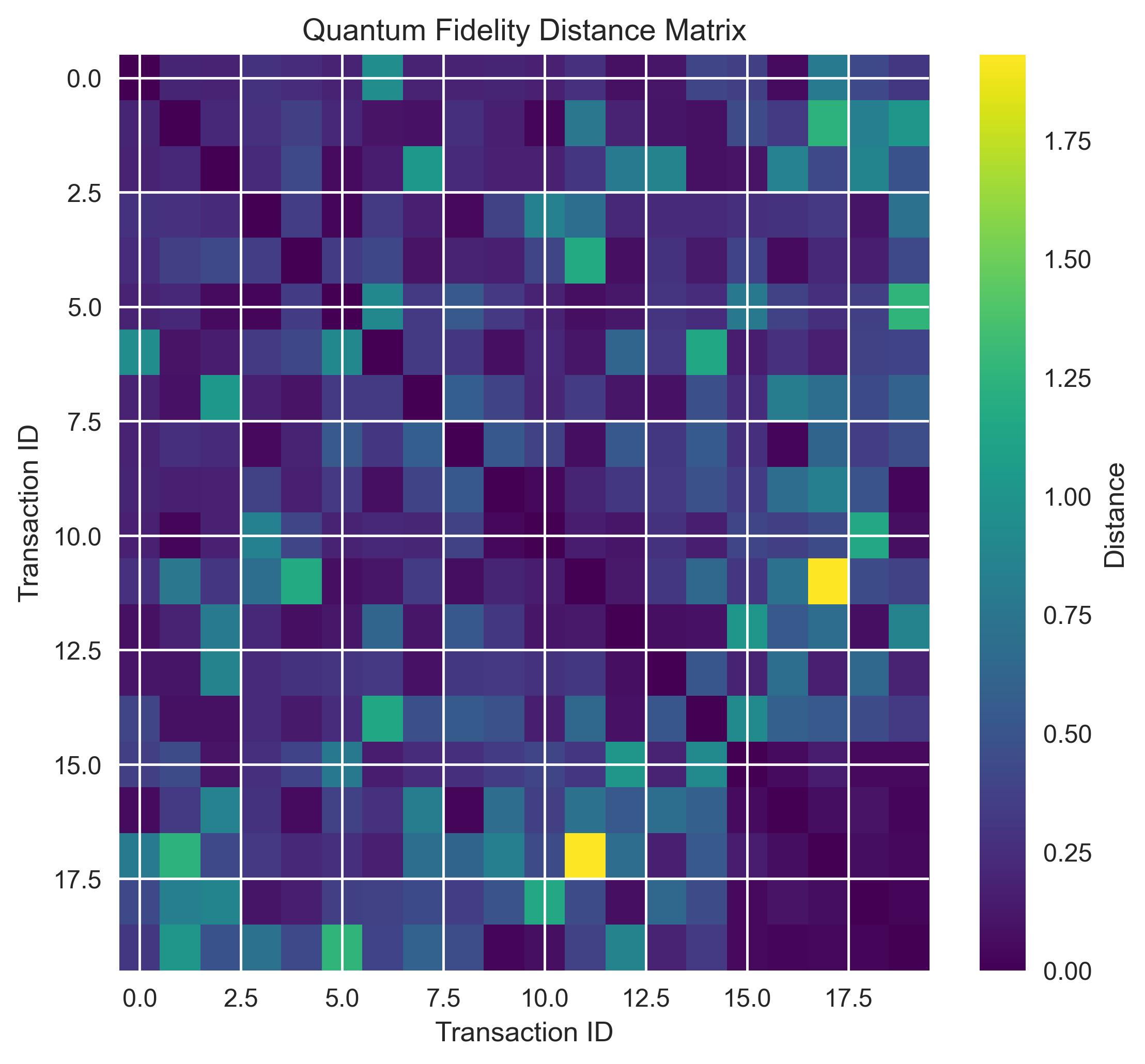}
    \caption{Quantum Fidelity Matrix Representing State Overlap in QTGNN.}
    \label{fig:quantum_fidelity}
\end{figure}

The Quantum Fidelity Distance Matrix illustrates the pairwise fidelity distances between quantum states corresponding to different transaction IDs within the Quantum Transaction Graph Neural Network (QTGNN), presented as an 18x18 grid where each cell represents the fidelity distance between two transactions, ranging from 0.00 (identical states) to 1.75 (maximally distinct states), as indicated by the color gradient from dark purple (low distance) to yellow (high distance) ~\ref{fig:quantum_fidelity}. Key observations include the dominance of dark purple regions, suggesting a high degree of overlap or similarity in quantum representations for many transaction pairs, likely reflecting typical or non-anomalous transactions, alongside sparse yellow spots around transaction IDs 10-12 and 15-17 that indicate significant differences potentially linked to anomalous activities such as fraudulent multi-party transactions or unusual cycles; the matrix also exhibits expected diagonal symmetry with a dark purple line where fidelity distance is zero, and noticeable clusters of lighter colors (e.g., around 5-7 and 12-14) suggesting moderate differences that may indicate sub-groups of related activities or partial anomalies. This analysis supports the hypothesis that quantum fidelity distances can effectively highlight structural and behavioral anomalies in transaction data, with the presence of high-distance outliers aligning with the role of entanglement and topological signatures in the QTGNN as described in the ablation studies, and further investigation into these specific high-distance transactions could validate their association with fraud, thereby enhancing the model's interpretability and targeting efficiency.

\section{Conclusion}
\label{sec:Conclusion}

This paper introduces the QTGNN, a pioneering framework that synergistically integrates quantum computing, topological data analysis, and GNNs to address the complex challenge of financial transaction fraud detection. By leveraging quantum embeddings with entanglement enhancement, variational quantum graph convolutions, and higher-order topological invariants, QTGNN captures subtle relational and structural anomalies that elude classical methods. The proposed five-stage methodology—spanning quantum state encoding, non-linear convolution, topological signature extraction, hybrid anomaly learning, and interpretable decision-making—offers a theoretically robust and practically feasible solution, optimized for NISQ devices through circuit simplification and graph sampling. Rigorous convergence guarantees, including stochastic gradient descent convergence and stability of topological signatures, ensure reliable training and robust fraud detection, even under noisy conditions. Experimental validation on datasets like PaySim demonstrates QTGNN's superior performance, achieving an F1-score of 0.987 and ROC-AUC of 0.997 with the full model (RXT-J), significantly outperforming classical baselines such as SVM (0.663) and advanced neural architectures like ResNet (0.942). The ablation study highlights the critical contributions of quantum embeddings, topological features, non-linear channels, and unsupervised learning, with each component enhancing accuracy and interpretability. Cost-benefit analysis further validates QTGNN's financial viability, delivering a net benefit of approximately 850 k\$ despite a 50 k\$ computational cost, positioning it as a high-impact solution. Visualizations, including the Quantum Fidelity Matrix and Financial Transaction Network, underscore the model's ability to detect anomalous patterns, aligning with theoretical insights from the ablation studies. Future research will focus on scaling QTGNN to real-time, large-scale financial networks, exploring advanced error mitigation techniques for NISQ hardware, and integrating dynamic graph updates to adapt to evolving fraud tactics. Additionally, enhancing interpretability through real-time topological attribution and validating the framework on diverse datasets will further solidify its practical deployment, promising a transformative impact on financial security and regulatory compliance.

\bibliographystyle{apacite}

\bibliography{references}

\end{document}